\documentclass[11pt]{article}
\usepackage[a4paper, margin=1in]{geometry}

\usepackage[usenames]{color}
\usepackage[figuresright]{rotating}
\usepackage{boxedminipage}
\usepackage[colorlinks,citecolor=blue,urlcolor=blue]{hyperref}
\usepackage{epstopdf}
\usepackage{natbib}
\usepackage{multirow}
\usepackage{enumerate}
\usepackage{enumitem}
\usepackage{verbatim}
\usepackage{bbm}
\usepackage{xcolor}
\usepackage{bm}
\usepackage{mathrsfs,color,dsfont}
\usepackage{algorithm,setspace}
\usepackage{algorithmic}
\usepackage{extarrows}
\usepackage{subcaption}
\usepackage{makecell}

\usepackage{tikz}
\usetikzlibrary{arrows}

\makeatletter
\pgfdeclareshape{datastore}{
	\inheritsavedanchors[from=rectangle]
	\inheritanchorborder[from=rectangle]
	\inheritanchor[from=rectangle]{center}
	\inheritanchor[from=rectangle]{base}
	\inheritanchor[from=rectangle]{north}
	\inheritanchor[from=rectangle]{north east}
	\inheritanchor[from=rectangle]{east}
	\inheritanchor[from=rectangle]{south east}
	\inheritanchor[from=rectangle]{south}
	\inheritanchor[from=rectangle]{south west}
	\inheritanchor[from=rectangle]{west}
	\inheritanchor[from=rectangle]{north west}
	\backgroundpath{
	\southwest \pgf@xa=\pgf@x \pgf@ya=\pgf@y
	\northeast \pgf@xb=\pgf@x \pgf@yb=\pgf@y
	\pgfpathmoveto{\pgfpoint{\pgf@xa}{\pgf@ya}}
	\pgfpathlineto{\pgfpoint{\pgf@xb}{\pgf@ya}}
	\pgfpathmoveto{\pgfpoint{\pgf@xa}{\pgf@yb}}
	\pgfpathlineto{\pgfpoint{\pgf@xb}{\pgf@yb}}
	}
}
\makeatother

\usepackage{amsthm,amsmath,amssymb,amsfonts, amsbsy, mathtools, epsfig}

\newcommand{\blind}{1}

\renewcommand{\baselinestretch}{1.31}
\newcommand{\linsp}{\renewcommand{\baselinestretch}{1.31}}
\newcommand{\linsps}{\renewcommand{\baselinestretch}{1.3}}

\newcommand{\Norm}[1]{\left\Vert #1\right\Vert}			
\newcommand{\vect}[1]{\boldsymbol{#1}}				

\newcommand{\tp}{{\top}}						
\newcommand{\card}{{\rm Card}}					
\newcommand{\diff}{\mathrm{d}}					
\newcommand{\argmin}[1]{\mathop{\rm{argmin}}_{#1}}	
											
\newcommand*\widebar[1]{\,						
	\hbox{%
   		\kern0.01em%
		\vbox{%
			\hrule height 0.5pt		
			\kern0.33ex			
			\hbox{%
				\kern-0.1em		
				\ensuremath{#1}%
				\kern-0.1em		
			}%
		}%
	}%
\,}%
\let\hat\widehat
\let\tilde\widetilde
\let\bar\widebar

\newcommand{\ie}{\mbox{\sl i.e.\;}}					

\newcommand{\mcD}{{\mathcal D}}

\newcommand{\mcH}{{\mathcal H}}					

\newcommand{\mcK}{{\mathcal K}}
\newcommand{\mcL}{{\mathcal L}}					

\newcommand{\mcN}{{\mathcal N}}					

\newcommand{\mcS}{{\mathcal S}}

\newcommand{\mcU}{{\mathcal U}}
\newcommand{\mcV}{{\mathcal V}}

\newcommand{\mbC}{{\mathbb C}}					

\newcommand{\mbE}{{\mathbb E}}					

\newcommand{\mbI}{{\mathbb I}}					

\newcommand{\mbP}{{\mathbb P}}					
\newcommand{\mbR}{{\mathbb R}}					
\newcommand{\mbS}{{\mathbb S}}					

\newcommand{\mfE}{{\mathfrak E}}					

\newcommand{\mfN}{{\mathfrak N}}					

\newcommand{\mfS}{{\mathfrak S}}

\newcommand{\mfX}{{\mathfrak X}}					

\DeclareMathAlphabet\EuScriptBF{U}{eus}{b}{n}

\newcommand{\supp}{\mathrm{supp}}	
\newcommand{\dissd}{\mathrm{DISSD}}	

\definecolor{DSgray}{cmyk}{0,1,0,0}

\definecolor{scolor}{cmyk}{0.5,2,0,0}

\newtheorem{lemma}{Lemma}
\newtheorem{theorem}{Theorem}
\newtheorem{corollary}{Corollary}
\newtheorem{proposition}{Proposition}
\newtheorem{remark}{Remark}
\newtheorem{example}{Example}
\newtheorem{assumptionmest}{Condition}
	
\newtheorem{assumptionglm}{Condition}
	
\newtheorem{assumptions}{Condition}

\newtheorem{assumptionns}{Condition}

\begin{document}


\linsps

\if1\blind
{
	\title{\bf Distributed Semi-Supervised Sparse Statistical Inference}
	 \author{Jiyuan Tu\thanks{
			School of Mathematics, Shanghai University of Finance and Economics, Shanghai, 200433, China (e-mail: tujy.19@gmail.com)}\,,
	 	Weidong Liu\thanks{
			School of Mathematical Sciences, MoE Key Lab of Artificial Intelligence, Shanghai Jiao Tong University, Shanghai 200240, China (e-mail: weidongl@sjtu.edu.cn).}\,,
		Xiaojun Mao\thanks{
			School of Mathematical Sciences, Ministry of Education Key Laboratory of Scientific and Engineering Computing, Shanghai Jiao Tong University, Shanghai 200240, China (e-mail: maoxj@sjtu.edu.cn).}\,,
		and Mingyue Xu
		\thanks{Department of Computer Science, Purdue University, West Lafayette, IN, 47906, USA (e-mail: xu1864@purdue.edu).}
  \date{}
}	\maketitle
} \fi

\if0\blind
{
	\bigskip
	\bigskip
	\bigskip
	\begin{center}
		{\LARGE\bf Distributed Semi-Supervised Sparse Statistical Inference}
	\end{center}
	\medskip
} \fi


\bigskip
\begin{abstract}
	The debiased estimator is a crucial tool in statistical inference for high-dimensional model parameters. However, constructing such an estimator involves estimating the high-dimensional inverse Hessian matrix, incurring significant computational costs. This challenge becomes particularly acute in distributed setups, where traditional methods necessitate computing a debiased estimator on every machine. This becomes unwieldy, especially with a large number of machines. In this paper, we delve into semi-supervised sparse statistical inference in a distributed setup. An efficient multi-round distributed debiased estimator, which integrates both labeled and unlabelled data, is developed. We will show that the additional unlabeled data helps to improve the statistical rate of each round of iteration. Our approach offers tailored debiasing methods for $M$-estimation and generalized linear models according to the specific form of the loss function. Our method also applies to a non-smooth loss like absolute deviation loss. Furthermore, our algorithm is computationally efficient since it requires only one estimation of a high-dimensional inverse covariance matrix.  We demonstrate the effectiveness of our method by presenting simulation studies and real data applications that highlight the benefits of incorporating unlabeled data.
\end{abstract}

\noindent
{\it Keywords}: Distributed learning; semi-supervised learning; debiased lasso; $M$-estimation; generalized linear model

\linsp


\section{Introduction}

High-dimensional models are prevalent in various scientific fields such as genomic analysis, image processing, signal processing, finance, and so on. However, traditional methods, such as maximum likelihood, least square regression, and principle component analysis break down due to the high dimensionality. To address this issue, new statistical methods and theories have been developed in high-dimensional statistics, which have played a critical role in advancing these fields. A common approach is to impose a sparsity structure on the data and model parameters. Many pioneering methods, such as $\ell_1$ regularization \citep{tibshirani1996regression}, non-convex regularization \citep{fan_li.2001, zhang2010jmlr, zhang.2010aos}, iterative hard thresholding \citep{blumensath_davies.2009, jain_tewari_kar.2014nips, yuan_li_zhang.2018jmlr} and debiasing techniques \citep{javanmard_montanari.2014jmlr, geer_buhlmann_etal.2014, zhang_zhang.2014, cai_cai_guo.2021jrssb} have been proposed to overcome these problems. 

Upon entering the era of big data, modern datasets exhibit more complex structures and larger scales, which pose additional challenges for data analysis. For example, in internet companies and hospitals, data are often stored across multiple agents independently.  However, direct data sharing may be prohibited due to privacy concerns, and communication between agents can be costly. As a result, distributed learning has emerged as a natural solution and has garnered significant attention in recent years. On the other hand, in some datasets, such as electronic health records and genetics data, the covariates are much more easily attainable than the labels. The need to utilize unlabeled data to improve the overall performance of statistical models naturally calls for research on semi-supervised learning.

In this paper, we aim to estimate the parameter $\vect{\beta}^*$ defined by the following population risk minimization problem
\begin{equation}	\label{eq:target_loss}
	\vect{\beta}^*=\argmin{\vect{\beta}\in\mbR^p}\mbE\big[f(\vect{X}^{\tp}\vect{\beta},Y)\big],
\end{equation}
in a distributed semi-supervised setup. This type of loss function encompasses a wide range of applications, such as Huber regression, logistic regression, quantile regression, and others. Specifically, we assume that the target parameter $\vect{\beta}^*$ is sparse, meaning that the number of nonzero elements of $\vect{\beta}^*$, $s$, is much less than the dimension $p$. While there has been extensive research on problem \eqref{eq:target_loss} in both low dimensional and high dimensional scenarios (see, e.g., \cite{javanmard_montanari.2014jmlr, geer_buhlmann_etal.2014, zhang_zhang.2014, huang2019distributed, battey2018distributed, jordan_etal.2019, fan_guo_etal.2019}), the statistical analysis in a distributed semi-supervised setup poses unique challenges and is scarce in the existing literature.

In the field of distributed learning,  there is a vast literature investigating the trade-off between communication efficiency and statistical accuracy. Among these methods, the divide-and-conquer (DC) approach stands out for its simplicity and practicality. This method involves executing the algorithm in parallel and combining the local estimators through averaging. The DC approach has been widely applied to various problems such as quantile regression \citep{volgushev_chao_cheng.2019aos, chen_liu_zhang.2019aos}, kernel ridge regression \citep{zhang_duchi_wainwright.2015jmlr, lin_guo_zhou.2017jmlr} and kernel density estimation \citep{li_lin_li.2013}. Although these one-shot methods offer high communication efficiency, it faces stringent restriction on the number of machines to obtain a minimax convergence rate. To alleviate these constraints, many iterative algorithms, such as distributed approximate Newton method \citep{shamir_etal.2013, zhang_lin.2015icml, jordan_etal.2019, fan_guo_etal.2019}, distributed primal-dual framework \citep{ smith_forte_etal.2017jmlr}, local stochastic gradient descent \citep{stich.2019iclr}, have been proposed. There are also some works studying the lower bound of communication complexity \citep{zhang_duchi_etal.2013nips, arjevani_shamir.2015nips}.

For high-dimensional distributed learning, naively averaging the penalized estimators would incur a large bias and would not improve over the local estimators. To address this issue, many works apply the divide-and-conquer method coupled with a well-known debiasing technique \citep{javanmard_montanari.2014jmlr, geer_buhlmann_etal.2014, zhang_zhang.2014, lee_liu_etal.2017, zhao_zhang_lian.TNNLS2020, lian_fan.2018} to achieve better performance. The debiasing technique is particularly useful when conducting hypothesis testing and constructing a confidence interval for high-dimensional parameters \citep{lee_liu_etal.2017, battey2018distributed, liu_xia_etal.2021}. In the classical distributed debiased estimator, each machine constructs a local estimator $\hat{\vect{\beta}}_j$, and computes the local debiased estimator
\begin{equation}	\label{eq:class_did}
	\hat{\vect{\beta}}^{\mathrm{db}}_j=\hat{\vect{\beta}}_j-\hat{\vect{\Theta}}_j(\hat{\vect{\beta}}_j)\frac{1}{n}\sum_{i\in\mcH_j}f'(\vect{X}_i^{\tp}\hat{\vect{\beta}}_j,Y_i)\vect{X}_i,
\end{equation}
where $\hat{\vect{\Theta}}_j(\hat{\vect{\beta}}_j)$ is an estimator of the inverse Hessian matrix $(\vect{H}(\hat{\vect{\beta}}_j))^{-1}=(\mbE[f''(\vect{X}^{\tp}\hat{\vect{\beta}}_j,Y)\vect{X}\vect{X}^{\tp}])^{-1}$, and $\mcH_j$ denotes the index set of the data on the $j$-th machine (where $j=1,...,m$). Then it takes the average of these estimators $\hat{\vect{\beta}}^{\mathrm{db}}_j$ as the final result. However, existing distributed debiased estimators are mostly one-shot, and unlike the single machine setup, a one-step debiasing approach may not be sufficient to eliminate the bias in a distributed setting, especially when the size of the local labeled sample is small, and the number of machines is large. Nevertheless, it is not straightforward to extend existing distributed debiased estimators to multi-round methods. Indeed, each debiasing step requires computing an estimator of the high-dimensional inverse Hessian matrix $(\vect{H}(\hat{\vect{\beta}}_j))^{-1}$, which is computationally prohibitive. Naively extending the algorithm to a multi-round procedure would incur a significant computational burden, rendering it impractical in practice. 

To address the aforementioned issues, we improve the classical distributed debiased estimators from two aspects: reducing the overall computational cost of the debiasing step and accelerating the per-round statistical rate. In traditional debiasing methods, the Hessian matrix $\vect{H}(\vect{\beta})$ varies with respect to the model parameter $\vect{\beta}$. In each round of iteration, we have to update the entire inverse Hessian matrix, which is computationally expensive. To reduce the overall computation burden of the multi-round debiasing approach, we propose debiasing methods for $M$-estimations and generalized linear models, respectively. Specifically, for the $M$-estimation problem, we propose to estimate the inverse Hessian matrix separately in two parts. For the generalized linear model, we weigh the gradients to eliminate the dependence of the Hessian matrix on the parameter $\vect{\beta}$. In both approaches, we only need to estimate one inverse covariance matrix $\vect{\Sigma}^{-1}=(\mbE\{\vect{X}\vect{X}^{\tp}\})^{-1}$ in the first round, which greatly facilitates computation. Our approach can also handle non-smooth loss functions like absolute deviation loss. Additionally, instead of constructing the debiased estimator on every local machine as in \eqref{eq:class_did}, our method aggregates gradients and constructs one debiased estimator on the first machine, which further saves computational resources.

On the other hand, to improve the convergence rate in each round and reduce the number of communication rounds, we propose a semi-supervised debiased estimator. Specifically, we leverage the unlabeled data to estimate the inverse Hessian matrix, which can significantly reduce bias when the unlabeled sample size is much larger than the locally labeled sample size. Semi-supervised learning has recently gained much interest in the statistical community, and several articles have provided semi-supervised inference frameworks that focus on estimating the mean of the response variable in both low dimensional \citep{zhang_brown_cai.2019} and high dimensional \citep{zhang_bradic.2021biometrika} scenarios. Additionally, in \cite{chakrabortty_cai.2018, azriel_brown_etal.JASA2021, deng_ning_etal.2020}, the authors considered semi-supervised inference for the miss-specified linear model in both low-dimensions and high-dimensions, while some works have studies inferring explained variance in high-dimensional linear regression \citep{cai_guo.2020} and evaluating the various prediction performance measures for the binary classifier \citep{gronsbell_cai.2018}. For distributed SSL, \cite{chang_lin_zhou.2017} studied kernel ridge regression and showed that the semi-supervised DC method benefits significantly from unlabeled data and allows more machines than the ordinary DC. Following this line of research, a bunch of works such as \cite{lin_zhou.2018, guo_shi_wu.jmlr2017, guo_etal.2019acha, hu_zhou.2021acha} have been proposed. However, these works all focused on one-shot non-parametric learning. To the best of our knowledge, there is no work considering distributed semi-supervised learning in high dimensions. 

In summary, our contribution is threefold:
\begin{itemize}
	\item Firstly, we propose a \underline{DI}stributed \underline{S}emi-\underline{S}upervised \underline{D}ebiased (DISSD) estimator that utilizes more unlabeled data to estimate the inverse Hessian matrix. We show that the additional unlabeled data helps to reduce the local bias and enhance the accuracy. In this paper, we mainly consider one machine having unlabeled data, which differs from the work of \cite{chang_lin_zhou.2017}. Our result can be easily extended to cases where more than one machine has unlabeled data (see Remark \ref{rem:multi_machine}).
	\item Secondly, for the $M$-estimation and generalized linear model, we propose two different debiasing techniques tailored to the specific form of the loss functions. Both approaches facilitate the use of unlabeled covariates information and help to reduce the computational cost in the multi-round debiasing step. Our method is much more computationally efficient than existing distributed debiased estimators like \cite{lee_liu_etal.2017, battey2018distributed, zhao_zhang_lian.TNNLS2020}.
	\item Lastly, simulation studies are carried out to demonstrate the necessity of multi-round debiasing in distributed learning, as well as the superiority of our semi-supervised method over the fully-supervised counterpart. We have numerically explored the impact of various hyperparameters on our model's performance. Moreover, we have conducted a comparative analysis with existing methods, showcasing the advantages of our method both in terms of statistical accuracy and computational time.
\end{itemize}

To better illustrate the advantage of our proposed method, we have incorporated a table to succinctly summarize the detailed comparisons between our work and existing methodologies. The results can be found in Table \ref{tab:compare} below.

\begin{table}[h]
	\centering
	\small
        \caption{Comparison of DISSD with existing methods.
	}\label{tab:compare}
	\bigskip
\begin{tabular}{c c c c c  }
\Xhline{2\arrayrulewidth}
	Method &\makecell{Statistical\\ accuracy}	& \makecell{Communication\\ complexity} 	& \makecell{Computational\\ complexity} &\makecell{Statistical\\ inference} \\ \Xhline{2\arrayrulewidth}
	\makecell{DC\\ \citep{battey2018distributed}} 	&\makecell{\vspace{1mm}$\sqrt{\frac{1}{mn}}+\frac{1}{n}$}	&$O(mp)$	&$O(mp^3)$ & Yes \\ \hline
        \makecell{Multi-round DC\\ naive extending \cite{battey2018distributed}}	&\makecell{\vspace{1mm}$\sqrt{\frac{1}{mn}}+(\frac{1}{n})^{(T+1)/2}$}	&$O(Tmp)$	&$O(Tmp^3)$ & Yes \\ \hline
	\makecell{CSL\\ \citep{jordan_etal.2019}}	&\makecell{\vspace{1mm}$\sqrt{\frac{1}{mn}}+\big(\frac{1}{n}\big)^{(T+1)/2}$}	&$O(Tmp)$	&$O(Tp^2+Tmp)$ &No	\\ \hline
	\makecell{CEASE\\ \citep{fan_guo_etal.2019}}	&\makecell{\vspace{1mm}$\sqrt{\frac{1}{mn}}+\big(\frac{1}{n}\big)^{(2T+1)/2}$}	&$O(Tmp)$	&$O(Tmp^2)$ &No	\\ \hline
	DISSD (this work)	&\makecell{\vspace{1mm}$\sqrt{\frac{1}{mn}}+\vect{\frac{1}{\sqrt{n}}\big(\frac{1}{n^*}\big)^{T/2}}$}	&$O(Tmp)$	&$O(\vect{p^3+Tmp})$ &\textbf{Yes} \\ \Xhline{2\arrayrulewidth}
\end{tabular}

\end{table}

Upon examination of the table, it becomes evident that our multi-round debiased estimator offers a notable reduction in computational burden when compared to the existing naive DC debiased estimator in \cite{battey2018distributed} (saves a factor of $m$ and $mT$ respectively). 
It's important to note that both \cite{jordan_etal.2019} and \cite{fan_guo_etal.2019} primarily address $\ell_1$-penalized loss functions, limiting their utility for statistical inference. To conduct statistical inference using the resultant estimators from these methods, one would need to perform one-step debiasing, incurring an additional computational overhead of $O(p^3)$. In contrast, the computational complexity of our approach has remainder term $O(Tmp)$, which is comparatively smaller than these two methods, offering a savings factor of $p$. 
Moreover, our method introduces a novel element by leveraging additional unlabeled data to expedite the distributed training process, which is a distinctive contribution to the literature. In summary, our methodology stands out by requiring the smallest total computational complexity while simultaneously delivering a statistically optimal estimator that facilitates statistical inference.


\subsection{Paper Organization and Notations}

The remainder of this paper is organized as follows. In Section \ref{sec:dss_ls}, we introduce the basic idea of our distributed debiased estimator starting from linear regression. In Section \ref{sec:dss_mest} and Section \ref{sec:dss_glm}, we propose debiasing methods for $M$-estimation and generalized linear model respectively. Theoretical results are presented in each section to demonstrate the acceleration effect of the semi-supervised methods. The simulation studies are given in Section \ref{sec:sim} to show how each factor affects the performance. Finally, we provide concluding remarks in Section \ref{sec:conclude}. All proofs of the theory are relegated to the Appendix.

For a vector $\vect{v}=(v_1,...,v_p)^T$, denote $|\vect{v}|_1=\sum_{l=1}^p|v_l|$, $|\vect{v}|_2=\sqrt{\sum_{l=1}^pv_l^2}$ and $|\vect{v}|_{\infty}=\sup_{1\leq l\leq p}|v_l|$. Moreover, we use $\supp(\vect{v}):=\{1\leq l\leq p\mid v_l\neq 0\}$ as the support of the vector $\vect{v}$. For a matrix $\vect{A}\in \mathbb{R}^{n\times p}$, define $\Norm{\vect{A}}_{L_1}:=\max_{1\leq j\leq p}\sum_{i=1}^{n}|a_{ij}|, |\vect{A}|_{\infty}:=\max_{1\leq i\leq n}\max_{1\leq j\leq p}|a_{ij}|$ as various matrix norms,  $\Lambda_{\max}(\vect{A})$ and $\Lambda_{\min}(\vect{A})$ as the largest and smallest eigenvalues of $\vect{A}$ respectively. For two sequences $a_n,b_n$, we say $a_n\asymp b_n$ if and only if $a_n=O(b_n)$ and $b_n=O(a_n)$ hold at the same time. Lastly, the generic constants are assumed to be independent of $m,n,$, and $p$.

Throughout the paper we assume that the $N=mn$ labeled pairs $\{(\vect{X}_i,Y_i)\}_{i=1}^N$ are evenly stored in $m$ machines $\mcH_1,\dots,\mcH_m$. We denote $\mcD_j$ as the index set of labeled data on the $j$-th machine. Moreover, on the first machine $\mcH_1$, there are additional $n^*-n$ unlabeled covariates $\{\vect{X}_i\}_{i\in\mcD^*_1}$, whose indices are collected in the set $\mcD_1^*$. Therefore there are $\mcH_1=\mcD_1\cup\mcD_1^*$, $|\mcD_j|=n$, $|\mcD_1^*|=n^*-n$ and $|\mcH_1|=n^*$. We call the first machine $\mcH_1$ as the master machine.


\section{Distributed Semi-Supervised Debiased Estimator for Sparse Linear Regression}	\label{sec:dss_ls}

In this section, we consider the distributed sparse linear regression with square loss as a toy example. We assume that the covariate $\vect{X}=(X_1,\dots,X_p)^{\tp}$ and the label $Y$ are generated from the following linear model 
\begin{equation}	\label{eq:lin_model}
	Y=\vect{X}^{\tp}\vect{\beta}^*+\epsilon,
\end{equation}
Throughout this section, we take the loss function as $f(x)=x^2/2$. Then the empirical loss function on the $j$-th machine becomes
\begin{equation}	\label{eq:ls_loss}
	\mcL_j(\vect{\beta}) = \frac{1}{2n}\sum_{i\in\mcD_j}(\vect{X}_i^{\tp}\vect{\beta}-Y_i)^2.
\end{equation} 
To conduct distributed learning, the classical way is to take the average of the debiased lasso estimator, which has been studied in \cite{lee_liu_etal.2017, battey2018distributed}. More precisely, let $\hat{\vect{\beta}}_j$ be the local Lasso estimator using the data on the $j$-th machine, and $\hat{\vect{\Omega}}_j$ be the local estimator of the precision matrix $\vect{\Omega}=(\mbE\{\vect{X}\vect{X}^{\tp}\})^{-1}$, then the one-shot debiased estimator is constructed as
\begin{equation}    \label{eq:classical_did}
    \bar{\vect{\beta}} = \frac{1}{m}\sum_{j=1}^m\left\{\hat{\vect{\beta}}_j - \hat{\vect{\Omega}}_j\nabla\mcL_j(\hat{\vect{\beta}}_j)\right\},
\end{equation}
where $\nabla\mcL_j(\hat{\vect{\beta}}_j )=\frac{1}{n}\sum_{i\in\mcD_j}\vect{X}_i\vect{X}_i^{\tp}\hat{\vect{\beta}}_j -\frac{1}{n}\sum_{i\in\mcD_j}Y_i\vect{X}_i^{\tp}$. While this method has been proven to enjoy a better convergence rate, there is still much room for improvement in our setting. On the one hand, it requires each machine to estimate the inverse covariance matrix $\vect{\Omega}$, which may waste much computational effort. On the other hand, we hope to incorporate the unlabeled data into the estimator. 

To address these difficulties, we propose an alternative estimator as follows. We first construct an initial estimator $\hat{\vect{\beta}}^{(0)}$ and send it to each machine, then each machine computes the local gradient $\nabla\mcL_j(\hat{\vect{\beta}}^{(0)})$ and sends to the master machine. Next, we construct the following estimator
\begin{equation}	\label{eq:dss_ls}
	\bar{\vect{\beta}}^{(1)} = \hat{\vect{\beta}}^{(0)} - \hat{\vect{\Omega}}_{\mcH_1}\frac{1}{m}\sum_{j=1}^m\nabla\mcL_j(\hat{\vect{\beta}}^{(0)} ),
\end{equation}
where $\hat{\vect{\Omega}}_{\mcH_1}$ denotes the consistent estimator of the precision matrix $\vect{\Omega}$ using all the covariates $\vect{X}_i$'s on the master machine $\mcH_1$. Compared with \eqref{eq:classical_did}, our method only requires estimating one inverse covariance matrix, therefore saving much computational cost. Moreover, it makes use of all covariates information on $\mcH_1$. As will be shown in Theorem \ref{thm:sm_conv}, the semi-supervised estimator constructed in \eqref{eq:dss_ls} has a better convergence rate than its fully supervised counterpart.

For the estimation of $\vect{\Omega}$ in high dimension, there are many works such as \cite{meinshausen_buhlmann.2006, cai_liu_luo.2011, yuan2010high, liu_luo.2015}. In this paper, we utilize the $\mathrm{SCIO}$ method proposed in \cite{liu_luo.2015}, which suggests to forming the estimator $\hat{\vect{\Omega}}_{\mcH_1}=(\hat{\vect{\omega}}_1,...,\hat{\vect{\omega}}_p)$ row by row. More specifically, the $l$-th row is given by minimizing the objective function
\begin{equation}	\label{eq:precision_est}
	\hat{\vect{\omega}}_l = \argmin{\vect{\omega}\in\mbR^p}\Big\{\frac{1}{2}\vect{\omega}^{\tp}\hat{\vect{\Sigma}}_{\mcH_1}\vect{\omega} - \vect{e}_l^{\tp}\vect{\omega}+\lambda_l|\vect{\omega}|_1\Big\},
\end{equation}
where $\vect{e}_l$ denotes the unit vector on the $l$-th coordinate, and $\lambda_l$ is the corresponding regularization parameter. The problem \eqref{eq:precision_est} can be solved by many efficient algorithms like FISTA \citep{beck2009fast} and ADMM \citep{boyd_etal.2011admm}, which makes our method more effective.

\begin{remark}
    Newton's method has been widely employed in the domain of distributed learning, spanning applications in low dimensional learning \citep{shamir_etal.2013, huang2019distributed, jordan_etal.2019, fan_liu_etal.2018}, sparse learning \citep{wang_kolar_etal.2017}, and kernel learning \citep{zhou.2020jmlr}. Our proposed method resembles the distributed approximate Newton method and is tailored for the high-dimensional debiased estimator, which is unexplored in the literature. Consequently, our work makes a valuable contribution to distributed high-dimension statistical inference. Further, the existing body of work in this area has predominantly centered on the fully-supervised setup, with limited exploration of semi-supervised scenarios. In this paper, we advance this research by introducing a substantial quantity of unlabeled data to enhance the estimation of Hessian information. Our findings provide theoretical evidence that the incorporation of unlabeled data effectively accelerates the distributed training process. It is conceivable that the methodologies elucidated in the aforementioned works can be readily extended to the semi-supervised framework, thereby yielding similar acceleration benefits. 
\end{remark}

\begin{remark}  \label{rem:multi_machine}
	By leveraging the idea of \cite{lee_liu_etal.2017}, our method can be easily extended to cases where more than one machine contains unlabeled data. Denote $\mcU$ as the index of machines containing unlabeled data. Then we can let the $j$-th machine (where $j\in\mcU$) to estimate $\lfloor p/|\mcU|\rfloor$ rows of $\vect{\Omega}$. The master machine sends the averaged gradient $m^{-1}\sum_{j=1}^m\nabla\mcL_j(\hat{\vect{\beta}}^{(0)})$ to each $\mcH_j$, where $j\in\mcU$. Suppose the $j$-th machine estimates the $l$-th row, then the $l$-th coordinate can be estimated as
	\begin{equation*}	
		\bar{\beta}^{(1)}_l = \hat{\beta}_l^{(0)} - \hat{\vect{\omega}}_{l}^{\tp}\frac{1}{m}\sum_{j=1}^m\nabla\mcL_j(\hat{\vect{\beta}}^{(0)} ),
	\end{equation*}
 	where $\hat{\vect{\omega}}_{l}$ is given by \eqref{eq:precision_est} using the covariates on $\mcH_j$. This helps to further reduce the computational burden as the estimation of $\hat{\vect{\omega}}_l$'s is processed in parallel. 
\end{remark}

In this section, we have presented the fundamental concept of our approach, which distinguishes itself from the existing distributed debiased estimators in two aspects. Firstly, we aggregate the gradients from all local machines, instead of the local estimators, and only estimate one high-dimensional inverse Hessian matrix; Secondly, we employ more unlabeled covariates to estimate the Hessian matrix to improve per-round accuracy. It should be noted that the Hessian matrix of the square loss is uncomplicated since it is independent of both the parameter $\vect{\beta}$ and the label $Y$. However, in the subsequent sections, we discuss $M$-estimation and generalized linear model, which have more complex forms. To eliminate the dependence of the Hessian matrix on other parameters, we present two customized techniques for the specific forms of the loss functions, which will be elaborated on in the following sections.


\section{Distributed Semi-Supervised Debiased Estimator for Sparse $M$-Estimation}	\label{sec:dss_mest}

In this section, we consider the distributed semi-supervised learning for a broader class of sparse $M$-estimation problems. In this model, the data are assumed to be generated from the linear model \eqref{eq:lin_model}. Based on different assumptions on the noise $\epsilon$, we should choose different loss functions $f(x)$. In general, the empirical loss function on the $j$-th machine is
\begin{equation}	\label{eq:mest_loss}
	\mcL_j(\vect{\beta}) = \frac{1}{n}\sum_{i\in\mcD_j}f(\vect{X}_i^{\tp}\vect{\beta}-Y_i).
\end{equation}
The situation becomes different from linear regression. To conduct the debiasing approach, a key step is to estimate the inverse of the population Hessian matrix, which is $\vect{H}(\vect{\beta}) = \mbE[f''(\vect{X}^{\tp}\vect{\beta}-Y)\vect{X}\vect{X}^{\tp}]$. Notice that $\vect{H}(\vect{\beta})$ contains the label information $Y$, therefore we cannot directly use the unlabeled covariates in the debiasing step. 

\subsection{Separated Estimation of Inverse Hessian}	\label{sec:separate_hess}

To address this issue, we observe that when $\vect{\beta}$ is close to the true parameter $\vect{\beta}^*$, the corresponding Hessian matrix satisfies
\begin{align*}
	\mbE\big[f''(\vect{X}^{\tp}\hat{\vect{\beta}}^{(0)}-Y)\vect{X}\vect{X}^{\tp}\big]&\approx \mbE\big[f''(\vect{X}^{\tp}\vect{\beta}^{*}-Y)\vect{X}\vect{X}^{\tp}\big]\\
	 &= \mbE\big[f''(\epsilon)\big]\mbE\big[\vect{X}\vect{X}^{\tp}\big],\stepcounter{equation}\tag{\theequation}\label{eq:pseudo_derivate}
\end{align*}
which is the multiplication of a scalar value and a matrix only involving the covariate $\vect{X}$. Therefore we can estimate the terms $\mbE\big[f''(\epsilon)\big]$  and $\mbE\big[\vect{X}\vect{X}^{\tp}\big]$ separately. More specifically, given a consistent estimator $\hat{H}^{(0)}$ of $\mbE\big[f''(\epsilon)\big]$, we can construct the following distributed semi-supervised debiased estimator
\begin{equation}	\label{eq:dss_mest}
	\bar{\vect{\beta}}^{(1)} = \hat{\vect{\beta}}^{(0)} - \frac{1}{\hat{H}^{(0)}}\hat{\vect{\Omega}}_{\mcH_1}\frac{1}{m}\sum_{j=1}^m\nabla\mcL_j(\hat{\vect{\beta}}^{(0)} ).
\end{equation}
The construction of $\hat{H}^{(0)}$ will be presented in the next part. Since the debiased estimator $\bar{\vect{\beta}}^{(1)}$ is dense, it does not have $\ell_2$-norm and $\ell_1$-norm consistency. To further encourage sparsity, we propose the thresholded estimator $\hat{\vect{\beta}}^{(1)}=(\hat{\beta}^{(1)}_1,...,\hat{\beta}_p^{(1)})^{\tp}$ where each coordinate is defined by
	\begin{equation*}
		\hat{\beta}_l^{(1)} = \bar{\beta}_l^{(1)}\cdot \mbI(|\bar{\beta}_l^{(1)}|\geq \tau).
	\end{equation*}
If we repeatedly take $\hat{\vect{\beta}}^{(1)}$ back to the above procedure, we can obtain an iterative algorithm, which is presented in Algorithm \ref{alg:dissd_mest}. It is worth noting that, by \eqref{eq:dss_mest} we know the estimator $\hat{\vect{\Omega}}_{\mcH_1}$ is unchanged throughout iterations. Therefore, we only need to estimate $\hat{\vect{\Omega}}_{\mcH_1}$ one time at the beginning of the procedure and keep it in the memory. At each iteration, it is enough to update the scalar $\hat{H}^{(t-1)}$ and the vector $\nabla\mcL_j(\hat{\vect{\beta}}^{(t-1)})$, which makes our algorithm much efficient.

\begin{algorithm}[!t]
	\caption{{\small Distributed Semi-Supervised Debiased (DISSD) estimator for $M$-estimation}}
	\label{alg:dissd_mest}
	\hspace*{\algorithmicindent} \hspace{-0.7cm}   {\textbf{Input:} Labeled data $\{(\vect{X}_i,Y_i)\mid i\in\mcD_j\}$ on worker machine $\mcH_j$ for $j=1,...,m$, and unlabeled data $\{\vect{X}_i\mid i\in\mcD^*_1\}$ on $\mcH_1$, the regularization parameter $\lambda_l$, the thresholding level $\tau$, the number of iterations $T$.} 	
	\begin{algorithmic}[1]
		\STATE The master machine constructs $\hat{\vect{\Omega}}_{\mcH_1}$ by \eqref{eq:precision_est} using all local covariates, and obtains the initial estimator $\hat{\vect{\beta}}^{(0)}$.
		\FOR{$t=1,\dots,T$}
		\STATE The master machine sends the parameter $\hat{\vect{\beta}}^{(t-1)}$.
 		\FOR{$j=1,\dots, m$}
		\STATE The $j$-th machine computes the local gradient
			\begin{equation*}
				\nabla\mcL_j(\hat{\vect{\beta}}^{(t-1)}) = \frac{1}{n}\sum_{i\in\mcH_j}f'(\vect{X}_i^{\tp}\hat{\vect{\beta}}^{(t-1)} - Y_i)\vect{X}_i,
			\end{equation*} 
			and the local estimator $\hat{H}_j^{(t-1)}$ by \eqref{eq:dissd_local_H_est}, then sends to the master machine.
		\ENDFOR
		\STATE The master machine takes $\hat{H}^{(t-1)}=m^{-1}\sum_{j=1}^m\hat{H}_j^{(t-1)}$ and lets
			\begin{equation*}
				\bar{\vect{\beta}}^{(t)} = \hat{\vect{\beta}}^{(t-1)} - \frac{1}{\hat{H}^{(t-1)}}\hat{\vect{\Omega}}_{\mcH_1}\frac{1}{m}\sum_{j=1}^m\nabla\mcL_j(\hat{\vect{\beta}}^{(t-1)} ).
			\end{equation*}
		\STATE The server obtain the thresholded estimator $\hat{\vect{\beta}}^{(t)}=(\hat{\beta}^{(t)}_1,...,\hat{\beta}_p^{(t)})^{\tp}$ where each coordinate is defined by
			\begin{equation*}
				\hat{\beta}_l^{(t)} = \bar{\beta}_l^{(t)}\cdot \mbI(|\bar{\beta}_l^{(t)}|\geq \tau).
			\end{equation*}
		\ENDFOR
	\end{algorithmic}
	 \textbf{Output:}  The final estimator $\hat{\vect{\beta}}^{(T)}$.
\end{algorithm}

For the choice of the initial parameter $\hat{\vect{\beta}}^{(0)}$, when the local sample size $n$ is sufficiently large, we can adopt the approach of minimizing the $\ell_1$-penalized fully supervised loss as formulated below:
\begin{equation}    \label{eq:init_est}
\hat{\vect{\beta}}^{(0)}=\argmin{\vect{\beta}\in\mbR^{p}}\frac{1}{n}\sum_{i\in\mcD_1}f(\vect{X}_i^{\tp}\vect{\beta}-Y_i)+\lambda_0|\vect{\beta}|_1.
\end{equation}
This optimization is carried out on the master machine $\mcH_1$. If the local labeled sample size $n$ is too small such that the estimator in \eqref{eq:init_est} does not fulfill the conditions on the initial rate in Theorem \ref{thm:sm_conv} and \ref{thm:nsm_conv}, one can apply the early-stopped distributed proximal gradient descent to obtain a better initial estimator. 

\paragraph{Construction of $\hat{H}^{(t-1)}$}

In \eqref{eq:pseudo_derivate}, we assume the loss function $f(\cdot)$ is twice differentiable for the ease of presentation. Indeed, we allow $f(\cdot)$ to be nonsmooth so that our theory embraces the case of median regression (see Example \ref{exp:med}). To achieve this goal, we first define the function $h(x)=\mbE\{f'(x+\epsilon)\}$. Then we know
\begin{equation*}	
    \vect{H}(\vect{\beta}^*) = \frac{\diff}{\diff x}\Big|_{x=0}\mbE\left[f'(x+\epsilon)\vect{X}\vect{X}^{\tp}\right] = h'(0)\mbE[\vect{X}\vect{X}^{\tp}].    
\end{equation*}
Therefore, to estimate $h'(0)$, we apply the same method as in Section 2.3 of \cite{tu2021byzantine}. More precisely,  assume $f'(\cdot)$ has finitely many distinct discontinuous points $x_1,\dots,x_K$, and it is differentiable outside of these points. We denote
\begin{equation*}
	\Delta_k=\lim_{x\rightarrow x_k^{+}}f'(x)-\lim_{x\rightarrow x_k^{-}}f'(x),
\end{equation*}
as the function gap at $x_k$, then $h'(0)$ can be explicitly written as
\begin{equation}	\label{eq:H_beta_def}
	h'(0) = \mbE\{f''(Y-\vect{X}^{\tp}\vect{\beta}^*)\}+\sum_{k=1}^K\Delta_k\rho(x_k),
\end{equation} 
where $f''(x)$ denotes the second-order derivative of $f(x)$, which can be defined almost everywhere on $\mbR$. Here $\rho(\cdot)$ denotes the density function of the noise $\epsilon$. Given a kernel function $\mcK(\cdot)$, a bandwidth $h_1$ and a consistent estimator $\hat{\vect{\beta}}^{(0)}$, on each local machine $\mcH_j$, $h'(0)$ can be locally estimated by
\begin{equation}	\label{eq:dissd_local_H_est}
\begin{aligned}
	\hat{H}_j^{(0)}:=&\frac{1}{n}\sum_{i\in\mcH_j}f''(Y_i-\vect{X}_i^{\tp}\hat{\vect{\beta}}^{(0)})+\sum_{k=1}^K\frac{\Delta_k}{nh_1}\sum_{i\in\mcH_j}\mcK\Big(\frac{Y_i-\vect{X}_i^{\tp}\hat{\vect{\beta}}^{(0)}-x_k}{h_1}\Big).
\end{aligned}
\end{equation}
Then the worker machine $\mcH_j$ sends the local estimator $\hat{H}_j^{(0)}$ to the master machine $\mcH_1$, and we take average over these local estimators $\hat{H}^{(0)}= \frac{1}{m} \sum_{j=1}^m \hat{H}_j^{(0)}$.
In the $t$-th iteration, we can estimate $\hat{H}^{(t-1)}$ in the same fashion.

\begin{example} \label{exp:huber}
	(Huber regression) Huber loss function is defined as
	\begin{equation*}
		\mcL(x)=
		\begin{cases}
			 x^2/2\quad &\text{for }|x|\leq\delta,\\
			 \delta |x|-\delta^2/2\quad &\text{otherwise}.
		\end{cases}
	\end{equation*}
	where $\delta$ is some pre-specified robustification parameter. In this case, the first-order derivative $\mcL'(x)$ is continuous, and we can compute that $\mcL''(x)=\mbI(|x|\leq\delta)$. Therefore $h'(0)=\mbP(|\epsilon|\leq\delta)$, and from \eqref{eq:dissd_local_H_est} we know each local estimator $\hat{H}_j^{(t-1)}$ can be constructed by
	\begin{equation*}
		\hat{H}_j^{(t-1)}=n^{-1}\sum_{i\in\mcH_j}\mbI\big(|Y_i-\vect{X}_i^{\tp}\hat{\vect{\beta}}^{(t-1)}|\leq\delta\big).
	\end{equation*}
\end{example}

\begin{example} \label{exp:med}
	(Median regression) In the median regression problem, the loss function is defined as $\mcL(x)=\frac{1}{2}|x|.$ Therefore its derivative $\mcL'(x)=1/2-\mbI(x\leq 0)$ is discontinuous at point $x_1=0$. We can easily obtain that $\mcL''(x)=0$ (for the discontinuous point $x_1=0$, we can directly define $\mcL''(0)=0$), and $\Delta_1=1$. Therefore from \eqref{eq:H_beta_def} we have $h'(0)=\rho(0)$. By equation \eqref{eq:dissd_local_H_est}, we can estimate $h'(0)$ by 
	\begin{align*}
		\hat{H}_j^{(t-1)}=&\frac{1}{nh_t}\sum_{i\in\mcH_j}\mcK\Big(\frac{Y_i-\vect{X}^{\tp}_i\hat{\vect{\beta}}^{(t-1)}}{h_t}\Big),
	\end{align*}
	on each local machine $\mcH_j$.
\end{example}

\begin{remark}
	Debiased estimators for linear models have been extensively studied in the literature \citep{javanmard_montanari.2014jmlr, zhang_zhang.2014, cai_cai_guo.2021jrssb}. However, for general loss functions, existing works typically recommend directly estimating the inverse Hessian matrix \citep{geer_buhlmann_etal.2014, battey2018distributed, zhao_zhang_lian.TNNLS2020}, which varies with the parameter $\vect{\beta}$. This approach hurdles the development of multi-round debiasing estimators, which are necessary in a distributed setup. Separated Hessian estimation has been introduced in the context of quantile regression \citep{zhao_kolar_liu.2014arxiv, chen_liu_etal.2020jmlr}, and extended to general loss functions in \cite{tu2021byzantine}. However, these previous works have focused primarily on transforming the loss function to a square loss to overcome the computational bottlenecks arising from the non-smoothness of the loss function, which differs from our motivation.
\end{remark}


\subsection{Theoretical Results of $\dissd$ for $M$-Estimation}   \label{sec:mest_theory}

In this section, we provide the theoretical results for the proposed methods. We first present several technical assumptions as follows. For ease of presentation, we denote $\mfE(X_0,\eta):=\mbE\big\{X_0^2\exp(\eta|X_0|)\big\},$
where $X_0$ is a random variable and $\eta>0$ is some fixed number.

\begin{assumptionmest}	\label{assump:X}
	There exists some constants $\kappa_X,C_X,C_M,C_q,\delta_{X}$ such that
        \begin{align*}
		&\sup_{\vect{v}\in\mbS^{p-1}}\mbE\left\{\exp(\kappa_X|\vect{v}^{\tp}\vect{X}|^2)\right\}\leq C_X,\quad \delta_X\leq\Lambda_{\min}(\vect{\Sigma})\leq\Lambda_{\max}(\vect{\Sigma})\leq \delta_X^{-1},\\
            &\|\Omega\|_{L_1}\leq C_M,\quad  \max_{1\leq l\leq p}\sum_{i=1}^p|\omega_{l,i}|^q\leq C_qs,\stepcounter{equation}\tag{\theequation}\label{eq:matrix_lp_sparse}
        \end{align*}
        where $\vect{\Omega} = (\omega_{i,l})_{i,l=1}^p$ and $0\leq q<1$.
\end{assumptionmest}

\begin{assumptionmest}	\label{assump:fpp_mfE}
	The loss function $f(\cdot)$ satisfies $\mbE\{\vect{X}f'(Y-\vect{X}^{\tp}\vect{\beta}^*)\}=0$. Furthermore, there exist constants $\kappa_f,C_f>0$ such that, for every pair of $(\vect{\beta}_1,\vect{\beta}_2)$, there is
	\begin{align*}
		&\mfE\Big\{\sup_{\vect{\beta}:|\vect{\beta}-\vect{\beta}_1|_2\leq |\vect{\beta}_1-\vect{\beta}_2|_2}\big|f''(Y-\vect{X}^{\tp}\vect{\beta}_1)-f''(Y-\vect{X}^{\tp}\vect{\beta})\big|,\kappa_f\Big\}\leq C_f|\vect{\beta}_1-\vect{\beta}_2|_2,\\
		&\mbE\big|f''(Y-\vect{X}^{\tp}\vect{\beta}_1)-f''(Y-\vect{X}^{\tp}\vect{\beta}_2)\big|\leq C_f\mbE\big|\vect{X}^{\tp}(\vect{\beta}_1-\vect{\beta}_2)\big|.
	\end{align*}
\end{assumptionmest}

\begin{assumptionmest}	\label{assump:lip_pdf}
	There exists a constant $C_d>0$ such that the noise $\epsilon$ has probability density function $\rho(\cdot)$ which satisfies
	\begin{equation*}
	\big|\rho(x)\big|\leq C_d,\quad \big|\rho(x)-\rho(y)\big|\leq C_d|x-y|,\quad\text{for any }x,y\in\mbR.
	\end{equation*}
	Moreover, there exists a constant $C_l$ such that $\min\{\rho(x_1),\dots \rho(x_K)\}>C_l>0$.
\end{assumptionmest}

\begin{assumptionmest}	\label{assump:kernel}
	The kernel function $\mcK(\cdot)$ satisfies $\int_{-\infty}^{\infty}\mcK(u)\diff u=1$ and $\mcK(u)=0$ for $|u|\geq1$. Moreover, $\mcK(\cdot)$ is Lipschitz continuous satisfying $|\mcK(x)-\mcK(y)|\leq C_k|x-y|$ holds for arbitrary $x,y\in\mbR$.
\end{assumptionmest}

\begin{assumptions}	\label{assump:smooth}
	(Smooth loss) There exists a constant $C_s>0$ such that, for any two points $x_1,x_2\in\mbR$, there is 
	\begin{align*}
		|f'(x_1)-f'(x_2)|\leq C_s|x_1-x_2|.
	\end{align*}
\end{assumptions}

\begin{assumptionns}	\label{assump:non-smooth}
	(Non-smooth loss) There exist constants $\kappa_{ns},C_{ns}>0$ such that, for every pair of $(\vect{\beta}_1,\vect{\beta}_2)$, there is
	\begin{align*}
		&\sup_{\vect{v}\in\mbS^{p-1}}\mfE\Big\{\sup_{\vect{\beta}:|\vect{\beta}-\vect{\beta}_1|_2\leq |\vect{\beta}_1-\vect{\beta}_2|_2}\big|\vect{v}^{\tp}\vect{X}f'(Y-\vect{X}^{\tp}\vect{\beta}_1)-\vect{v}^{\tp}\vect{X}f'(Y-\vect{X}^{\tp}\vect{\beta})\big|,\kappa_{ns}\Big\}\leq C_{ns}|\vect{\beta}_1-\vect{\beta}_2|_2,\\
		&\sup_{\vect{v}\in\mbS^{p-1}}\mbE\big|\vect{v}^{\tp}\vect{X}f'(Y-\vect{X}^{\tp}\vect{\beta}_1)-\vect{v}^{\tp}\vect{X}f'(Y-\vect{X}^{\tp}\vect{\beta}_2)\big|\leq C_{ns}|\vect{\beta}_1-\vect{\beta}_2|_2.
	\end{align*}
\end{assumptionns}

Condition \ref{assump:X} is some regularity conditions on the covariate $\vect{X}$. In particular, we assume the inverse covariance matrix is $\ell_q$-sparse (where $0\leq q<1$), a condition frequently utilized as an alternative to the strict sparse condition (see, e.g., \cite{bickel_levina.2008aos, cai_liu_luo.2011, raskutti_etal.2011tit}). Condition \ref{assump:fpp_mfE} assumes the continuity of the second-order derivative $f''(\cdot)$ in a weak sense. This condition also appears in \cite{tu2021byzantine}. Condition \ref{assump:lip_pdf} and \ref{assump:kernel} assumes the probability density function $\rho(\cdot)$ and $\mcK(\cdot)$ to be regular. Clearly, they are used only when the loss function $f'(\cdot)$ is not second-order differentiable. We impose different conditions for smooth loss (Condition \ref{assump:smooth}) and non-smooth loss (Condition \ref{assump:non-smooth}). As will be seen in Appendix \ref{sec:exp_verif}, the Huber regression model in Example \ref{exp:huber} and the median regression model in Example \ref{exp:med} satisfy these conditions given the covariate $\vect{X}$ admits sub-gaussian property.

Given the above conditions, we can prove the following convergence rate for both smooth and non-smooth losses.

\begin{theorem}	\label{thm:sm_conv}
	(Convergence rate for smooth loss) Under Conditions \ref{assump:X}, \ref{assump:fpp_mfE} and \ref{assump:smooth}, assume the initial estimator $\hat{\vect{\beta}}^{(0)}$ satisfies $|\hat{\vect{\beta}}^{(0)}-\vect{\beta}^*|_2=O(r_n)$ and $|\hat{\vect{\beta}}^{(0)}|_0\leq C_rs$. There are rate constraints $s=o(\min\{(mn)/\log p, (n^*/\log p)^{(1-q)/3}\})$, $r_n=o(1/\sqrt{s})$. Moreover, take $\lambda_l\asymp \sqrt{\frac{\log p}{n^*}}$ in \eqref{eq:precision_est}, and the threshold level $\tau$ satisfies $\tau\asymp \sqrt{\frac{\log p}{mn}} +r_ns\big(\frac{\log p}{n^*}\big)^{(1-q)/2}+r_n^2$ for some $C_{\tau}>0$.  Then the one-step distributed debiased estimator $\hat{\vect{\beta}}^{(1)}$ satisfies
	\begin{equation}	\label{eq:smooth_rate}
	\begin{aligned}
		\big|\hat{\vect{\beta}}^{(1)}-\vect{\beta}^*\big|_{\infty}=&O_{\mbP}\left(\sqrt{\frac{\log p}{mn}} + r_n\sqrt{\frac{s\log p}{mn}}+r_ns\Big(\frac{\log p}{n^*}\Big)^{(1-q)/2}+r_n^2\right),\\
		\big|\hat{\vect{\beta}}^{(1)}-\vect{\beta}^*\big|_{2}=&O_{\mbP}\left(\sqrt{\frac{s\log p}{mn}}+r_n\sqrt{\frac{s^2\log p}{mn}}+r_ns^{3/2}\Big(\frac{\log p}{n^*}\Big)^{(1-q)/2}+\sqrt{s}r_n^2\right),\\
		\big|\hat{\vect{\beta}}^{(1)}-\vect{\beta}^*\big|_{1}=&O_{\mbP}\left(s\sqrt{\frac{\log p}{mn}}+r_ns\sqrt{\frac{s\log p}{mn}}+r_ns^2\Big(\frac{\log p}{n^*}\Big)^{(1-q)/2}+sr_n^2\right).
	\end{aligned}
	\end{equation}
\end{theorem}

\begin{theorem}	\label{thm:nsm_conv}
	(Convergence rate for non-smooth loss) Under Conditions \ref{assump:X} to \ref{assump:kernel} and \ref{assump:non-smooth}, assume the initial estimator $\hat{\vect{\beta}}^{(0)}$ satisfies $|\hat{\vect{\beta}}^{(0)}-\vect{\beta}^*|_2=O(r_n)$ and $|\hat{\vect{\beta}}^{(0)}|_0\leq C_rs$. There are rate constraints $s=o(\min\{(mn)/\log p, (n^*/\log p)^{(1-q)/3}\})$, $r_n=o(1/\sqrt{s})$. Moreover, take $\lambda_l\asymp \sqrt{\frac{\log p}{n^*}}$ in \eqref{eq:precision_est}, and $
	    \tau\asymp \sqrt{\frac{\log p}{mn}} +r_ns\big(\frac{\log p}{n^*}\big)^{(1-q)/2}+r_n^2$, $
	    h_1\asymp r_n$.
	Then the one-step distributed debiased estimator $\hat{\vect{\beta}}^{(1)}$ satisfies
	\begin{equation}	\label{eq:non-smooth_rate}
	\begin{aligned}
		\big|\hat{\vect{\beta}}^{(1)}-\vect{\beta}^*\big|_{\infty}=&O_{\mbP}\left(\sqrt{\frac{\log p}{mn}} + \sqrt{\frac{r_ns\log p}{mn}}+r_ns\Big(\frac{\log p}{n^*}\Big)^{(1-q)/2}+r_n^2\right),\\
		\big|\hat{\vect{\beta}}^{(1)}-\vect{\beta}^*\big|_{2}=&O_{\mbP}\left(\sqrt{\frac{s\log p}{mn}}+\sqrt{\frac{r_ns^2\log p}{mn}}+r_ns^{3/2}\Big(\frac{\log p}{n^*}\Big)^{(1-q)/2}+\sqrt{s}r_n^2\right),\\
		\big|\hat{\vect{\beta}}^{(1)}-\vect{\beta}^*\big|_{1}=&O_{\mbP}\left(s\sqrt{\frac{\log p}{mn}}+s\sqrt{\frac{r_ns\log p}{mn}}+r_ns^2\Big(\frac{\log p}{n^*}\Big)^{(1-q)/2}+sr_n^2\right).
	\end{aligned}
	\end{equation}
\end{theorem}

As we can see, the convergence rate of $\hat{\vect{\beta}}^{(1)}$ for smooth loss and non-smooth loss only differs in the second term in \eqref{eq:smooth_rate} and \eqref{eq:non-smooth_rate} respectively. When $r_n=o_{\mbP}(s^{-1/2})$, the second term is always dominated by the first term, therefore could be eliminated. From the third term of \eqref{eq:smooth_rate} and \eqref{eq:non-smooth_rate}, we can see that the convergence rate is refined by a factor $s^{3/2}\big(\log p/n^*\big)^{(1-q)/2}$, which is better than $s^{3/2}\big(\log p/n\big)^{(1-q)/2}$ in the fully supervised case. The thresholding parameter $\tau$ is suggested to have the same order as $\sqrt{\frac{\log p}{mn}} +r_ns\big(\log p/n^*\big)^{(1-q)/2}+r_n^2$, which depends on the unknown initial rate $r_n$. In practice, we can use $\tau = C_{\tau}(\sqrt{\log p/(mn)} +r_ns\big(\log p/n^*\big)^{(1-q)/2}+r_n^2)$ by letting $r_n = \sqrt{s\log p/n}$ if $\hat{\vect{\beta}}^{(0)}$ is obtained from \eqref{eq:init_est}. The constant $C_{\tau}$ should be tuned carefully. As a corollary, we can show the convergence rate of the multi-round estimator $\hat{\vect{\beta}}^{(t)}$.

\begin{corollary}	\label{cor:betaT_conv}
	Under assumptions in Theorem \ref{thm:sm_conv} or Theorem \ref{thm:nsm_conv} then we have that the $T$-th round $\mathrm{DISSD}$ estimator $\hat{\vect{\beta}}^{(T)}$ satisfies
    \begin{equation}    \label{eq:multi_rate}
	\begin{aligned}
		\big|\hat{\vect{\beta}}^{(T)}-\vect{\beta}^*\big|_{\infty}=&O_{\mbP}\left(\sqrt{\frac{\log p}{mn}}+r_ns^{(3T-1)/2}\Big(\frac{\log p}{n^*}\Big)^{(1-q)T/2}+s^{-1}(\sqrt{s}r_n)^{2^T}\right),\\
		\big|\hat{\vect{\beta}}^{(T)}-\vect{\beta}^*\big|_{2}=&O_{\mbP}\left(\sqrt{\frac{s\log p}{mn}}+r_ns^{3T/2}\Big(\frac{\log p}{n^*}\Big)^{(1-q)T/2}+s^{-1/2}(\sqrt{s}r_n)^{2^T}\right),\\
		\big|\hat{\vect{\beta}}^{(T)}-\vect{\beta}^*\big|_{1}=&O_{\mbP}\left(s\sqrt{\frac{\log p}{mn}}+r_ns^{(3T+1)/2}\Big(\frac{\log p}{n^*}\Big)^{(1-q)T/2}+(\sqrt{s}r_n)^{2^T}\right).
	\end{aligned}
    \end{equation}
\end{corollary}

In order to attain a near-optimal convergence rate $\sqrt{s\log p/(mn)}$, we require the last two terms in \eqref{eq:multi_rate} to be dominated by the first term. Given that the third term converges at a super-exponential rate, our primary concern centers on the second term. More specifically, we seek conditions under which $(s^3(\log p/n^*)^{1-q})^{T/2}=O(\sqrt{\log p/(mn)})$ holds. By taking logarithm to both sides and rearranging the terms, we have that
\begin{equation}    \label{eq:iter_num}
    T\geq \frac{\log(mn)-\log s-\log\log p}{(1-q)\log n^*-(1-q)\log\log p-3\log s}.
\end{equation}
In the fully-supervised case where $n=n^*$, the required number of iterations should satisfy \eqref{eq:iter_num} with $n^*$ replaced by $n$. Consequently, when $n^*$ significantly exceeds $n$, our supervised algorithm demands notably fewer iterations to achieve the near-optimal rate. This observation is also validated through our experimental results (see Section \ref{sec:sim}).

Next, we prove the asymptotic normality result for our $\mathrm{DISSD}$ estimator.

\begin{theorem}	\label{thm:normality}
	Under assumptions in Corollary \ref{cor:betaT_conv}, and the iteration number $T$ is large enough such that $|\hat{\vect{\beta}}^{(T)}-\vect{\beta}^*|_2=O_{\mbP}(\sqrt{s\log p/(mn)})$. Further assume the rate constraints $s = o\Big( \frac{(n^*)^{(1-q)/3}}{(\log p)^{(2-q)/3}},\frac{\sqrt{mn}}{\log p}\Big)$, then we have that
	\begin{equation*}
		\frac{\sqrt{mn}}{\sigma_l}(\bar{\beta}^{(T)}_l-\beta^*_l)\xrightarrow{d}\mcN(0,1),
	\end{equation*}
	where
	\begin{equation*}
		\sigma_l^2 = \frac{\mbE\big[\big\{f'(\epsilon)\big\}^2\big]}{\{h'(0)\}^2}\Omega_{l,l}.
	\end{equation*}
\end{theorem}
As we can see, to guarantee asymptotic normality, the sparsity level should satisfy\\ $s=o((n^*)^{(1-q)/3}/(\log p)^{(2-q)/3})$, which is looser than the fully supervised case (where\\ $s=o(n^{(1-q)/3}/(\log p)^{(2-q)/3})$).


\section{Distributed Semi-Supervised Debiased Estimator for Sparse GLM}	\label{sec:dss_glm}

In this section, we consider distributed semi-supervised learning for generalized linear models (GLM). In a GLM, given a function $\psi:\mbR\rightarrow\mbR$, the observations $(\vect{X},Y)\in\mbR^{p+1}$ are generated according to the following conditional probability function
\begin{equation}	\label{eq:glm_gen}
	\mbP(Y\mid\vect{X}) = \tilde{c}\exp\left\{\frac{Y\vect{X}^{\tp}\vect{\beta}^*-\psi(\vect{X}^{\tp}\vect{\beta}^*)}{c(\sigma)} \right\},
\end{equation} 
where $\tilde{c}$ and $c(\sigma)$ are some scale constants, and $\vect{\beta}^*$ is the true model parameter. The corresponding empirical loss function on the $j$-th machine is
\begin{equation}	\label{eq:glm_loss}
	\mcL_j(\vect{\beta})= \frac{1}{n}\sum_{i\in\mcD_j}f(Y_i,\vect{X}_i^{\tp}\vect{\beta})= \frac{1}{n}\sum_{i\in\mcD_j}\Big\{-Y_i\vect{X}_i^{\tp}\vect{\beta}+\psi(\vect{X}_i^{\tp}\vect{\beta})\Big\}.
\end{equation} 
We note that the population Hessian matrix at $\vect{\beta}^*$,
\begin{equation}	\label{eq:glm_pop_hess}
	\vect{H}(\vect{\beta}^*)=\mbE[\nabla^2|_{\vect{\beta}=\vect{\beta}^*}f(Y,\vect{X}^{\tp}\vect{\beta})]=\mbE[\psi''(\vect{X}^{\tp}\vect{\beta}^*)\vect{X}\vect{X}^{\tp}],
\end{equation} 
does not contain the label $Y$. Therefore, to construct a DSS debiased estimator, a direct way is to use all covariate information to estimate the inverse Hessian matrix \eqref{eq:glm_pop_hess}. However, the Hessian matrix $\vect{H}(\vect{\beta})$ is dependent on the model parameter $\vect{\beta}$. When performing a multi-round debiasing approach, we need to update the parameter $\vect{\beta}$ and estimate the inverse Hessian matrix recursively, which is computationally expensive. 

\subsection{Accelerated Debiasing by Weighted Gradients}	\label{sec:weight_grad}

To alleviate the computational burden and accelerate the multi-round debiasing algorithm, we propose a gradient-weighting approach as follows. To motivate our construction, we first consider the generic form of the one-step debiased estimator. More precisely, for each coordinate $l\in\{1,\dots,p\}$, there is
\begin{equation*}
	\tilde{\beta}_l^{(1)} = \hat{\beta}_l^{(0)} - \vect{u}_l^{\tp}\frac{1}{mn}\sum_{j=1}^m\sum_{i\in\mcD_j}w_i\Big\{-Y_i\vect{X}_i+\psi'(\vect{X}_i^{\tp}\hat{\vect{\beta}}^{(0)})\vect{X}_i\Big\},
\end{equation*}
where $\hat{\vect{\beta}}^{(0)}$ is an initial estimator, $w_i$'s are data-dependent gradient weight, and $\vect{u}_l$'s are the projection directions. Applying Taylor expansion for $\psi'$ at $\vect{X}_i^{\tp}\hat{\vect{\beta}}^{(0)}$, we can decompose the error of $\tilde{\vect{\beta}}_l^{(1)}- \vect{\beta}^*$ as follows
\begin{align*}
	\tilde{\beta}_l^{(1)}-\beta_l^* =& \vect{u}_l^{\tp}\frac{1}{mn}\sum_{j=1}^m\sum_{i\in\mcD_j}w_i\vect{X}_i\epsilon_i+\Big(\vect{e}_l^{\tp} - \vect{u}_l^{\tp}\frac{1}{mn}\sum_{j=1}^m\sum_{i\in\mcD_j}w_i\psi''(\vect{X}_i^{\tp}\hat{\vect{\beta}}^{(0)})\vect{X}_i\vect{X}^{\tp}\Big)(\hat{\vect{\beta}}^{(0)} - \vect{\beta}^* )\\
		&+\vect{u}_i^{\tp}\frac{1}{mn}\sum_{j=1}^m\sum_{i\in\mcD_j}w_i\vect{X}_i\Delta_i, \stepcounter{equation}\tag{\theequation}\label{eq:glm_dbdecomp}
\end{align*}
where $\epsilon_i=\psi'(\vect{X}_i^{\tp}\vect{\beta}^*) -Y_i$, $\Delta_i = \psi'''(\tilde{\vect{\beta}})\{\vect{X}_i^{\tp}(\hat{\vect{\beta}}^{(0)}-\vect{\beta}^*)\}^2$ for some $\tilde{\vect{\beta}}$ lying between $\vect{\beta}^*$ and $\hat{\vect{\beta}}^{(0)}$, and $\vect{e}_l$ is the canonical basis of $\mbR^p$. From this decomposition, to conduct statistical inference for the debiased estimator, we require the second term dominated by the first term in \eqref{eq:glm_dbdecomp}. Therefore, the projection direction $\vect{u}_l$ should approximate the $l$-th column of the inverse of the matrix $\hat{\vect{H}} = \frac{1}{mn}\sum_{j=1}^m\sum_{i\in\mcD_j}w_i\psi''(\vect{X}_i^{\tp}\hat{\vect{\beta}}^{(0)})\vect{X}_i\vect{X}^{\tp}$. However, to construct $\vect{u}_l$ for every coordinate, we should solve $p$ high-dimensional optimization problems successively, which incurs heavy computational overhead. To reduce the computational burden, we can take $w_i = 1/\psi''(\vect{X}_i^{\tp}\hat{\vect{\beta}}^{(0)})$ so that the matrix $\hat{\vect{H}}$ is independent of the parameter $\hat{\vect{\beta}}^{(0)}$.

In summary, given an initial parameter $\hat{\vect{\beta}}^{(0)}$, we compute the following weighted gradient
\begin{equation}	\label{eq:glm_rewgrad}
	\vect{g}_j(\hat{\vect{\beta}}^{(0)}) = \frac{1}{n}\sum_{i\in\mcD_j}\frac{1}{\psi''(\vect{X}_i^{\tp}\hat{\vect{\beta}}^{(0)})}\Big\{-Y_i\vect{X}_i+\psi'(\vect{X}_i^{\tp}\hat{\vect{\beta}}^{(0)})\vect{X}_i\Big\},
\end{equation}
and construct the distributed semi-supervised debiased estimator
\begin{equation}	\label{eq:dss_glm_tilde}
	\bar{\vect{\beta}}^{(1)} = \hat{\vect{\beta}}^{(0)} - \hat{\vect{\Omega}}_{\mcH_1}\frac{1}{m}\sum_{j=1}^m\vect{g}_j(\hat{\vect{\beta}}^{(0)}).
\end{equation} 
Compared with directly estimate the inverse of $\vect{H}(\hat{\vect{\beta}}^{(0)})$ in  \eqref{eq:glm_pop_hess}, this approach only estimate the inverse covariance matrix $\vect{\Omega}$, which is independent of the parameter $\hat{\vect{\beta}}^{(0)}$. Therefore we only need to estimate $\hat{\vect{\Omega}}_{\mcH_1}$ for one time, which saves much computational power when performing multi-round algorithm.

To obtain a sparse estimator, we threshold each coordinate and get $\hat{\vect{\beta}}^{(1)}=(\hat{\beta}^{(1)}_1,...,\hat{\beta}_p^{(1)})^{\tp}$ where 
	\begin{equation*}
		\hat{\beta}_l^{(1)} = \bar{\beta}_l^{(1)}\cdot \mbI(|\bar{\beta}_l^{(1)}|\geq \tau).
	\end{equation*}
The iterative algorithm is presented in Algorithm \ref{alg:dissd_glm}. 

\begin{algorithm}[!t]
	\caption{{\small Distributed Semi-Supervised Debiased (DISSD) estimator for generalized linear model}}
	\label{alg:dissd_glm}
	\hspace*{\algorithmicindent} \hspace{-0.7cm}   {\textbf{Input:} Labeled data $\{(\vect{X}_i,Y_i)\mid i\in\mcD_j\}$ on worker machine $\mcH_j$ for $j=1,...,m$, and unlabeled data $\{\vect{X}_i\mid i\in\mcD^*_1\}$ on $\mcH_1$, the regularization parameter $\lambda_l$, the thresholding level $\tau$, the number of iterations $T$.} 	
	\begin{algorithmic}[1]
		\STATE The master machine constructs $\hat{\vect{\Omega}}_{\mcH_1}$ by \eqref{eq:precision_est} using all local covariates, and obtains the initial estimator $\hat{\vect{\beta}}^{(0)}$.
		\FOR{$t=1,\dots,T$}
		\STATE The master machine sends the parameter $\hat{\vect{\beta}}^{(t-1)}$.
 		\FOR{$j=1,\dots, m$}
		\STATE The $j$-th machine computes the local gradient
			\begin{equation*}
				\vect{g}_j(\hat{\vect{\beta}}^{(t-1)}) = \frac{1}{n}\sum_{i\in\mcD_j}\frac{1}{\psi''(\vect{X}_i^{\tp}\hat{\vect{\beta}}^{(t-1)})}\Big\{-Y_i\vect{X}_i+\psi'(\vect{X}_i^{\tp}\hat{\vect{\beta}}^{(t-1)})\vect{X}_i\Big\},
			\end{equation*} 
			then sends it to the master machine.
		\ENDFOR
		\STATE The master machine takes
			\begin{equation*}
				\bar{\vect{\beta}}^{(t)} = \hat{\vect{\beta}}^{(t-1)} - \hat{\vect{\Omega}}_{\mcH_1}\frac{1}{m}\sum_{j=1}^m\vect{g}_j(\hat{\vect{\beta}}^{(t-1)}).
			\end{equation*}
		\STATE The server obtain the thresholded estimator $\hat{\vect{\beta}}^{(t)}=(\hat{\beta}^{(t)}_1,...,\hat{\beta}_p^{(t)})^{\tp}$ where each coordinate is defined by
			\begin{equation*}
				\hat{\beta}_l^{(t)} = \bar{\beta}_l^{(t)}\cdot \mbI(|\bar{\beta}_l^{(t)}|\geq \tau).
			\end{equation*}
		\ENDFOR
	\end{algorithmic}
	 \textbf{Output:}  The final estimator $\hat{\vect{\beta}}^{(T)}$.
\end{algorithm}

Similarly to the approach detailed in Algorithm \ref{alg:dissd_mest}, the choice of the initial parameter $\hat{\vect{\beta}}^{(0)}$ can be made by minimizing the local $\ell_1$-penalized loss as in \eqref{eq:init_est}, or through the implementation of early-stopped proximal gradient descent. The decision between these methods is contingent upon the specific local labeled sample size.

\begin{example} \label{exp:logit}
    (Logistic regression) In the logistic regression model, we generate the labels from the following distribution
    \begin{equation}	\label{eq:logit_gen}
	\begin{cases}
		\mbP(Y=1|\vect{X}) = \frac{\exp(\vect{X}^{\tp}\vect{\beta}^*)}{1+\exp(\vect{X}^{\tp}\vect{\beta}^*)},\\
		\mbP(Y=0|\vect{X}) = \frac{1}{1+\exp(\vect{X}^{\tp}\vect{\beta}^*)}.
	\end{cases}
    \end{equation}
    Then the function $\psi(\cdot)$ is defined as
    \begin{equation*}
        \psi(x) = \log(1+\exp(x)).
    \end{equation*}
    We can readily compute that $\psi'(x)=1/(1+\exp(-x))$ and $\psi''(x)=1/\{(1+\exp(x))(1+\exp(-x))\}$.
\end{example}

\begin{remark}
	In this section, we have demonstrated the significance of the gradient-weighting technique in reducing the computational cost of implementing the multi-round debiasing approach. 
	The gradient-weighting technique has been applied in contemporary works for various purposes, such as simultaneous testing \citep{ma_cai_li.2021jasa}, case probability inference \citep{guo_rakshit_etal.2021jmlr}, genetic relatedness inference \citep{ma_guo_etal.2022arXiv}, and confidence interval construction \citep{cai_guo_ma.2021jasa} for logistic regression models. Different weights were proposed for different applications. However, these studies primarily focus on a single-machine setup, and the rationale for weighting the gradient differs from ours. Furthermore, while all aforementioned works only consider logistic models, our method extends to a broader class of generalized linear models.
\end{remark}


\subsection{Theoretical Results of $\dissd$ for GLM}   \label{sec:glm_theory}

In this section, we provide the theories of $\dissd$ for GLM. We first list some technical conditions as follows.

\begin{assumptionglm}	\label{assumpglm:X}
	There exists some constants $\kappa_X,C_X,C_M,C_q,\delta_{X}$ such that
        \begin{align*}
		&\sup_{\vect{v}\in\mbS^{p-1}}\mbE\left\{\exp(\kappa_X|\vect{v}^{\tp}\vect{X}|^2)\right\}\leq C_X,\quad \delta_X\leq\Lambda_{\min}(\vect{\Sigma})\leq\Lambda_{\max}(\vect{\Sigma})\leq \delta_X^{-1},\\
            &\|\Omega\|_{L_1}\leq C_M,\quad  \max_{1\leq l\leq p}\sum_{i=1}^p|\omega_{l,i}|^q\leq C_qs,\stepcounter{equation}\tag{\theequation}\label{eq:matrix_lp_sparse}
        \end{align*}
        where $\vect{\Omega} = (\omega_{i,l})_{i,l=1}^p$ and $0\leq q<1$.
\end{assumptionglm}

\begin{assumptionglm}	\label{assumpglm:subgauss}
	There exists some constants $\kappa_0,C_0$ such that
	\begin{equation*}
		\max\Big\{\sup_{\vect{v}\in\mbS^{p-1}}\mbE\left\{\exp(\kappa_0|\vect{v}^{\tp}\vect{X}|^2)\right\},\mbE\big[\exp(\kappa_0|\psi'(\vect{X}^{\tp}\vect{\beta}^*)-Y|^2)\big] \Big\}\leq C_0.
	\end{equation*}
\end{assumptionglm}

\begin{assumptionglm}	\label{assumpglm:lip_link}
	The function $\psi(\cdot)$ is three-time differentiable, and there exists a uniform constant $c_{\psi}>0$ such that
	\begin{equation*}
		\sup_{x\in\mbR}\max\{|\psi''(x)|,|\psi'''(x)|\}\leq c_{\psi}.
	\end{equation*}
\end{assumptionglm}

\begin{assumptionglm}	\label{assumpglm:bound_below}
	With probability not less than $1-p^{-\gamma}$, there holds $\min\{\psi''(\vect{X}_i^{\tp}\vect{\beta}^*)\}\geq c_l$ for $i\in\mcD_j$ ($1\leq j\leq m$) and some small constant $c_l>0$.
\end{assumptionglm}

Condition \ref{assumpglm:X} is the same as Condition \ref{assump:X} in Section \ref{sec:mest_theory}. Condition \ref{assumpglm:subgauss} assumes subgaussianity of the gradient $\{-Y_i\vect{X}_i+\psi'(\vect{X}_i^{\tp}\vect{\beta}^*)\vect{X}_i\}$. In Condition \ref{assumpglm:lip_link}, we assume the function $\psi$ to be smooth with bounded derivatives. Condition \ref{assumpglm:bound_below} requires $\vect{X}_i^{\tp}\vect{\beta}^*$ to be bounded for all $1\leq i\leq N$ with high probability, which is a common assumption in the literature. As will be seen in Appendix \ref{sec:exp_verif}, the logistic regression model in Example \ref{exp:logit} satisfies these conditions given the covariate $\vect{X}$ is uniformly bounded.

With the above conditions, we can prove the following convergence rate.

\begin{theorem}	\label{thm:glm_conv}
	Under Conditions \ref{assumpglm:X} to \ref{assumpglm:bound_below}, assume the initial estimator $\hat{\vect{\beta}}^{(0)}$ satisfies $|\hat{\vect{\beta}}^{(0)}-\vect{\beta}^*|_2=O(r_n)$ and $|\hat{\vect{\beta}}^{(0)}|_0\leq C_rs$. There are rate constraints $s=o(\min\{(mn)/\log p, (n^*/\log p)^{(1-q)/3}\})$, $r_n=o(1/(s^{3/2}\log^2p))$. Moreover, take $\lambda_l\asymp \sqrt{\frac{\log p}{n^*}}$ in \eqref{eq:precision_est}, and the threshold level $\tau$ satisfies $\tau\asymp C_{\tau}(\sqrt{\frac{\log p}{mn}} +r_ns\big(\frac{\log p}{n^*}\big)^{(1-q)/2}+s\log^2pr_n^2)$ for some $C_{\tau}>0$.  Then the one-step distributed debiased estimator $\hat{\vect{\beta}}^{(1)}$ satisfies
	\begin{equation}	\label{eq:glm_rate}
	\begin{aligned}
		\big|\hat{\vect{\beta}}^{(1)}-\vect{\beta}^*\big|_{\infty}=&O_{\mbP}\left(\sqrt{\frac{\log p}{mn}} + r_ns\Big(\frac{\log p}{n^*}\Big)^{(1-q)/2} + r_n\sqrt{\frac{s\log p}{mn}} + s\log^2pr_n^2\right),\\
		\big|\hat{\vect{\beta}}^{(1)}-\vect{\beta}^*\big|_{2}=&O_{\mbP}\left(\sqrt{\frac{s\log p}{mn}} + r_ns^{3/2}\Big(\frac{\log p}{n^*}\Big)^{(1-q)/2} + r_n\sqrt{\frac{s^2\log p}{mn}} + s^{3/2}\log^2pr_n^2\right),\\
		\big|\hat{\vect{\beta}}^{(1)}-\vect{\beta}^*\big|_{1}=&O_{\mbP}\left(s\sqrt{\frac{\log p}{mn}} + r_ns^{2}\Big(\frac{\log p}{n^*}\Big)^{(1-q)/2} + r_ns\sqrt{\frac{s\log p}{mn}} + s^2\log^2pr_n^2\right).
	\end{aligned}
	\end{equation}
\end{theorem}

We can observe that the convergence rate of $\hat{\vect{\beta}}^{(1)}$ for GLM is similar to that of M-estimator in Theorem \ref{thm:sm_conv}, except for the last term. More specifically, the last term of GLM has an additional factor $s\log^2 p$, which is caused by the weighting procedure. Compared with Theorem \ref{thm:sm_conv} and \ref{thm:nsm_conv}, the theory for GLM requires more stringent conditions on the initial rate $r_n$, which can be regarded as the price of weighting the gradients. Next, we present the convergence rate of the multi-round estimator $\hat{\vect{\beta}}^{(T)}$.

\begin{corollary}	\label{cor:glmT_conv}
	Under assumptions in Theorem \ref{thm:sm_conv} or Theorem \ref{thm:nsm_conv} then we have that the $T$-th round $\mathrm{DISSD}$ estimator $\hat{\vect{\beta}}^{(T)}$ satisfies
	\begin{equation}	\label{eq:glm_multi_rate}
	\begin{aligned}
		\big|\hat{\vect{\beta}}^{(T)}-\vect{\beta}^*\big|_{\infty}=&O_{\mbP}\left(\sqrt{\frac{\log p}{mn}} + r_ns^{(3T-1)/2}\Big(\frac{\log p}{n^*}\Big)^{(1-q)T/2} + \frac{1}{s^2\log^2p}\big(s^{3/2}\log^2pr_n)^{2^T}\right),\\
		\big|\hat{\vect{\beta}}^{(T)}-\vect{\beta}^*\big|_{2}=&O_{\mbP}\left(\sqrt{\frac{s\log p}{mn}} + r_ns^{3T/2}\Big(\frac{\log p}{n^*}\Big)^{(1-q)T/2} + \frac{1}{s^{3/2}\log^2p}\big(s^{3/2}\log^2pr_n)^{2^T}\right),\\
		\big|\hat{\vect{\beta}}^{(T)}-\vect{\beta}^*\big|_{1}=&O_{\mbP}\left(s\sqrt{\frac{\log p}{mn}} + r_ns^{(3T+1)/2}\Big(\frac{\log p}{n^*}\Big)^{(1-q)T/2} + \frac{1}{s\log^2p}\big(s^{3/2}\log^2pr_n)^{2^T}\right).
	\end{aligned}
	\end{equation}
\end{corollary}

Similarly as in Corollary \ref{cor:betaT_conv}, achieving a near-optimal statistical rate entails the dominance of the first term over the second term in \eqref{eq:glm_multi_rate}. Consequently, the number of iterations required, denoted as $T$, should align with the conditions established in \eqref{eq:iter_num}. Next, we prove the asymptotic normality result for our $\mathrm{DISSD}$ estimator for GLM.

\begin{theorem}	\label{thm:glm_normality}
	Under assumptions in Corollary \ref{cor:glmT_conv}, and the iteration number $T$ is large enough such that $|\hat{\vect{\beta}}^{(T)}-\vect{\beta}^*|_2=O_{\mbP}(\sqrt{\log p/(mn)})$. Further assume the rate constraints $s = o\big((mn)^{1/4}/\log^{3/2}p$, $ (n^*)^{(1-q)/3}/(\log^{(2-q)/3}p)\big)$, then we have that
	\begin{equation*}
		\frac{\sqrt{mn}}{\sigma_l}(\bar{\beta}^{(T)}_l-\beta^*_l)\xrightarrow{d}\mcN(0,1),
	\end{equation*}
	where
	\begin{equation*}
		\sigma_l^2 =  \vect{\omega}_l^{\tp}\mbE\Big[\frac{1}{\psi''(\vect{X}^{\tp}\vect{\beta}^*)}\vect{X}\vect{X}^{\tp}\Big]\vect{\omega}_l.
	\end{equation*}
\end{theorem}
Compared with Theorem \ref{thm:normality}, we have a more strict constraint on the sparsity level, namely, $s = o((mn)^{1/4}/\log^{3/2}p)$. This is also caused by the gradient-weighting technique.


\section{Simulation Study}\label{sec:sim}

{In the empirical analysis, we conduct two classes of experiments to demonstrate the effectiveness of our method. The first part examines our proposed $\dissd$ method on synthetic data. The latter is an application to the corresponding sparse linear regression task.}


\subsection{Experiments on the Synthetic Data}	\label{sec:ssoe_simu}

In this section, we show the performance of the distributed semi-supervised debiased estimator on the synthetic dataset.

\paragraph{Parameter Settings.} Throughout the experiments, we fix dimension $p=500$ and the true coefficient 
\begin{equation*}
	\vect{\beta}^*=(1,(s-1)/s,...,1/s,\vect{0}_{p-s}^{\tp})^{\tp}.
\end{equation*}
We fix the sparsity level to be $s=10$, unless specified otherwise. The covariate $\vect{X}$ is sampled from $\mcN(\vect{0},\vect{\Sigma})$, where $\vect{\Sigma}^{-1}$ is a block diagonal matrix whose block size is $5$ and each block has off-diagonal entries equal to $0.5$ and diagonal $1$. We repeat 100 independent simulations and report the averaged estimation error and the corresponding standard error.

To illustrate the performance of our method, we mainly compare six methods 
\begin{itemize}
	\item Local Lasso: Solve the Lasso problem using only the data on $\mcH_1$;
	\item Pooled Lasso: Collect all data together and solve the Lasso problem;
	\item DISSD(X): Semi-supervised distributed debiased estimator. Namely perform Algorithm \ref{alg:dissd_mest} or \ref{alg:dissd_glm} with unlabeled sample size fixed as $X$.
        \item CSL: Communication-efficient surrogate likelihood framework proposed in \cite{jordan_etal.2019}, which solves
            \begin{equation}    \label{eq:csl}
                \vect{\beta}^{(t+1)}_{\mathrm{CSL}} = \argmin{\vect{\beta}}\Big\{\mcL_0(\vect{\beta}) - \Big\langle \nabla\mcL_0(\hat{\vect{\beta}}^{(t)}_{\mathrm{CSL}})-\frac{1}{m}\sum_{j=1}^m\nabla\mcL_j(\hat{\vect{\beta}}^{(t)}_{\mathrm{CSL}}),\vect{\beta}\Big\rangle+\lambda_t|\vect{\beta}|_1\Big\}
            \end{equation}
            in each iteration.
\end{itemize}

Several other distributed algorithms have been proposed for high-dimensional learning, including CEASE from \cite{fan_guo_etal.2019} and Debias-DC from \cite{battey2018distributed}. However, due to the extensive processing time these methods require and the lack of substantial improvement in statistical error, we have opted not to include a comparison with them in our synthetic data experiments. Instead, we conduct a thorough comparative analysis using a real dataset.

\subsubsection{Huber Regression} 

In the Huber regression problem, we generate the noise $\epsilon$ from the mixture of normal distributions $0.9\mcN(0,1)+0.1\mcN(0,100)$. More precisely, with probability $0.9$, the value of $\epsilon$ is distributed according to $\mcN(0,1)$ and is otherwise drawn from a $\mcN(0,100)$ distribution.  For the choice of robustification parameter $\delta$, we follow the classical literature \citep{huber2004robust} and take $\delta=1.345$. In Table \ref{tab:huber_mnstar}, we maintain a fixed local labeled sample size of $n=100$ and a dimension of $p=500$, while varying the number of machines $m$ in $\{20,50,100\}$ and adjusting the size of the unlabeled sample size $n^*$ within $\{0,150,450\}$. The reported results include the $\ell_2$-error and $F_1$-score for both 1-Step DISSD and 5-Step DISSD. The table clearly demonstrates the consistent superiority of the semi-supervised estimator over its fully supervised counterpart. Additionally, the $\ell_2$ error diminishes with increasing unlabeled sample size, highlighting the positive impact of the larger unlabeled dataset. Furthermore, the multi-round $\dissd$ method consistently diminishes estimation error.

\begin{table}[h]
	\centering
	\small
	\caption{The $\ell_2$-errors and $F_1$-score and their standard errors (in parentheses) of one-step $\dissd$ and five-step $\dissd$,  under labeled local sample size $n=100$.The parameter dimension is $p=500$. Noises are generated from the mixed normal distribution $0.9\mcN(0,1)+0.1\mcN(0,100)$ and the loss function is chosen as Huber loss with robustification parameter $1.345$.
	}\label{tab:huber_mnstar}
	\bigskip
\begin{tabular}{c| c|cc| cc  }
\hline
\multirow{2}{*}{$m$}	&\multirow{2}{*}{$n^*$}	&\multicolumn{2}{c|}{ 1-Step DISSD }		&\multicolumn{2}{c}{ 5-Step DISSD }\\
	&	&$\ell_2$-error	&$F_1$-score	&$\ell_2$-error	&$F_1$-score\\ \hline
	&100	&1.227(0.061)	&0.705(0.008)	&0.363(0.093)	&0.788(0.040)	\\ 
20	    &250    &1.025(0.053)	&0.697(0.006)	&0.287(0.047)	&0.773(0.033)	\\
	&550	&0.919(0.049)	&0.689(0.003)	&0.315(0.047)	&0.750(0.020)	\\ \hline
	&100    &1.155(0.061)	&0.836(0.043)	&0.204(0.056)	&0.987(0.023)	\\ 
50	    &250    &0.888(0.056)	&0.801(0.031)	&0.077(0.023)	&0.993(0.018)	\\
	&550    &0.678(0.049)	&0.757(0.019)	&0.064(0.020)	&0.985(0.025)	\\ \hline
	&100    &1.149(0.063)	&0.949(0.045)	&0.196(0.060)	&0.994(0.017)	\\ 
100	    &250    &0.870(0.057)	&0.961(0.037)	&0.060(0.018)	&1.000(0.000)	\\
	&550	&0.633(0.047)	&0.927(0.040)	&0.039(0.010)	&1.000(0.000)\\ \hline
\end{tabular}
\end{table}

In Table \ref{tab:huber_ps}, we maintain a fixed labeled sample size of $n=100$, set the unlabeled sample size to be $n^*-n=450$, and keep the number of machines constant at $m=100$. We then vary the dimension $p$ across $\{200,400,600\}$ and adjust the sparsity level $s$ within $\{5,10,20\}$. The results clearly demonstrate that both larger values of $p$ and $s$ correspond to increased statistical errors and a decline in the $F_1$-score. This aligns with our theoretical expectations outlined in Corollary \ref{cor:betaT_conv}.  

\begin{table}[h]
	\centering
	\small
	\caption{The $\ell_2$-errors and $F_1$-score and their standard errors (in parentheses) of one-step $\dissd$ and five-step $\dissd$,  under labeled local sample size $n=100$.The unlabeled sample size is $n^*-n=450$. Noises are generated from the standard mixed normal distribution $0.9\mcN(0,1)+0.1\mcN(0,100)$ and the loss function is chosen as Huber loss with robustification parameter 1.345.
	}\label{tab:huber_ps}
	\bigskip
\begin{tabular}{c| c|cc| cc  }
\hline
\multirow{2}{*}{$p$}	&\multirow{2}{*}{$s$}	&\multicolumn{2}{c|}{ 1-Step DISSD }		&\multicolumn{2}{c}{ 5-Step DISSD }\\
	&	&$\ell_2$-error	&$F_1$-score	&$\ell_2$-error	&$F_1$-score\\ \hline
	&5	&0.427(0.049)	&1.000(0.000)	&0.027(0.009)	&1.000(0.000)	\\ 
200	    &10 &0.580(0.053)	&0.965(0.031)	&0.037(0.011)	&1.000(0.000)	\\
	&20	&0.847(0.058)	&0.740(0.009)	&0.078(0.015)	&0.935(0.031)	\\ \hline
	&5  &0.459(0.046)	&1.000(0.000)	&0.025(0.009)	&1.000(0.000)	\\ 
400	    &10 &0.636(0.051)	&0.954(0.036)	&0.039(0.012)	&1.000(0.000)	\\
	&20 &0.932(0.055)	&0.708(0.004)	&0.098(0.017)	&0.889(0.033)	\\ \hline
	&5  &0.456(0.051)	&1.000(0.000)	&0.028(0.012)	&1.000(0.000)	\\ 
600	    &10 &0.623(0.052)	&0.922(0.045)	&0.039(0.010)	&1.000(0.000)	\\
	&20	&0.966(0.049)	&0.695(0.002)	&0.113(0.017)	&0.858(0.029)\\ \hline
\end{tabular}
\end{table}

In Figure \ref{fig:huber_iter}, the dimension is set at $p=500$, and we vary the local labeled sample size $n$ among $\{50,100,200\}$, while adjusting the unlabeled sample size $n^*-n$ in $\{0,150,450\}$. The figure presents the $\ell_2$-error and processing time for these methods. It is evident from the figure that the inclusion of unlabeled data significantly contributes to error reduction. Particularly noteworthy is the observation that when the local labeled sample size is too small, the fully supervised distributed debiased estimator may diverge. It is clear that our DISSD method outperforms the CSL method in terms of statistical accuracy. Specifically, each step of CSL requires approximately ten times the processing time of DISSD (after computing the inverse Hessian matrix). Additionally, CSL necessitates the tuning of the parameter $\lambda_t$ in \eqref{eq:csl} at each step, further augmenting its runtime.

\begin{figure}[h]
	\begin{center}
		\includegraphics[width=1\textwidth]{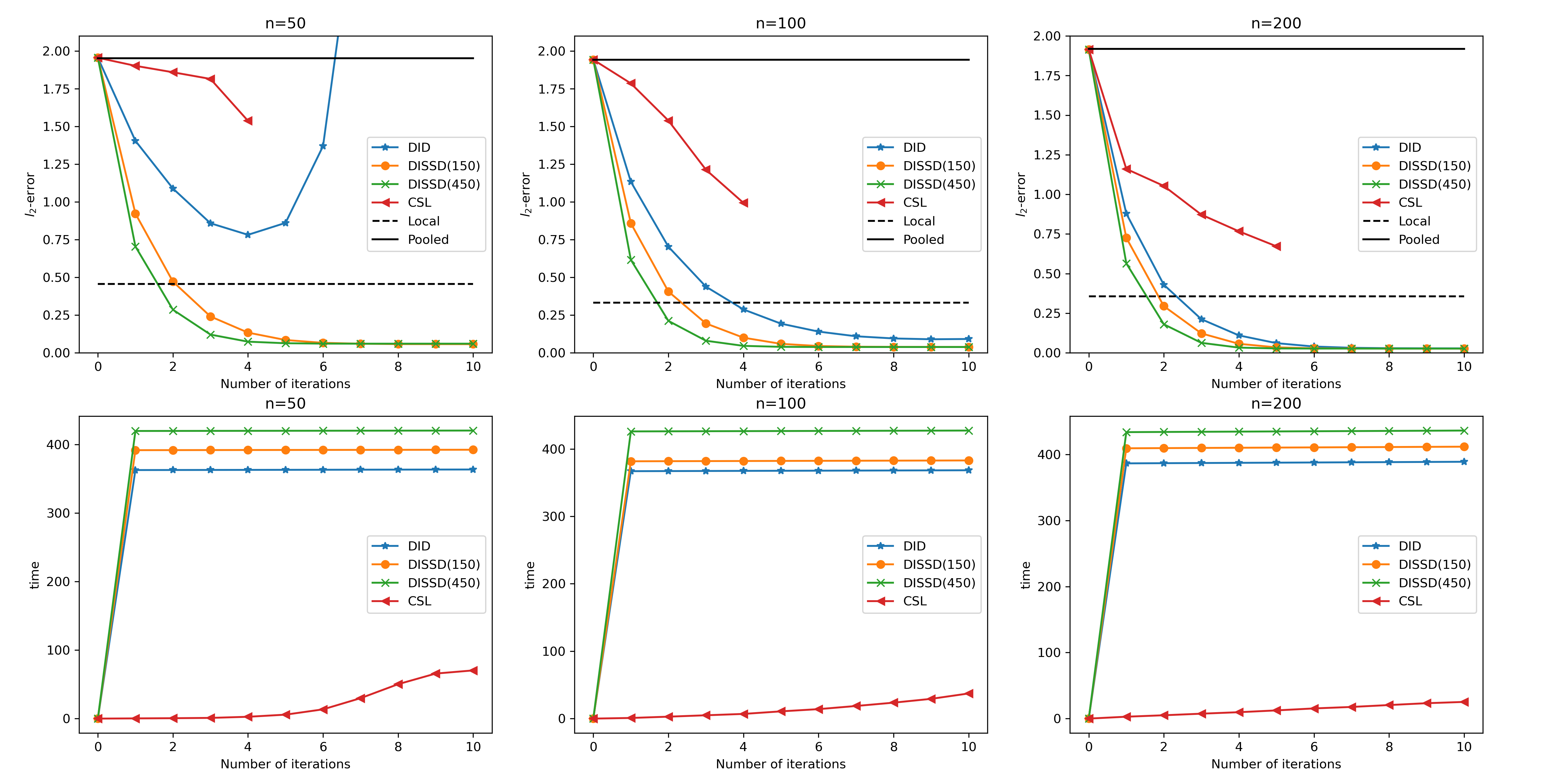}
	\end{center}
	\caption{The $\ell_2$-error (the first row) and the processing time (the second row) over the number of iterations in Huber regression. The local labeled sample size takes value in $\{50,100,200\}$, the local unlabeled sample size takes value in $\{0,150,450\}$, the number of machine is $100$, and the dimension $p$ is $500$}
	\label{fig:huber_iter}
\end{figure}

In Figure \ref{fig:huber_initial}, we investigate the efficacy of our method under three distinct initializations: local estimator, random initialization, and 10-round distributed proximal gradient descent (PGD). The dimension is set at $p=500$, and we vary the local labeled sample size $n$ among $\{50,100,200\}$, while taking the unlabeled sample size $n^*-n=150$. The figure presents the $\ell_2$-error for these methods. The results reveal a clear trend: when the local labeled sample size is relatively large, all three initializations in the DISSD method yield consistent estimators. However, in scenarios where the local labeled sample size is limited, DISSD with random initialization exhibits divergence as the iteration number increases. Notably, the early-stopped distributed proximal gradient descent outperforms other methods, thereby expediting the training process. It is essential to note that while PGD provides superior performance, it necessitates additional communication compared to the local estimator. Consequently, the choice of initialization should consider both statistical accuracy and communication constraints. In cases where communication costs are a concern, employing the local estimator for initialization is a prudent choice. 

\begin{figure}[h]
	\begin{center}
		\includegraphics[width=1\textwidth]{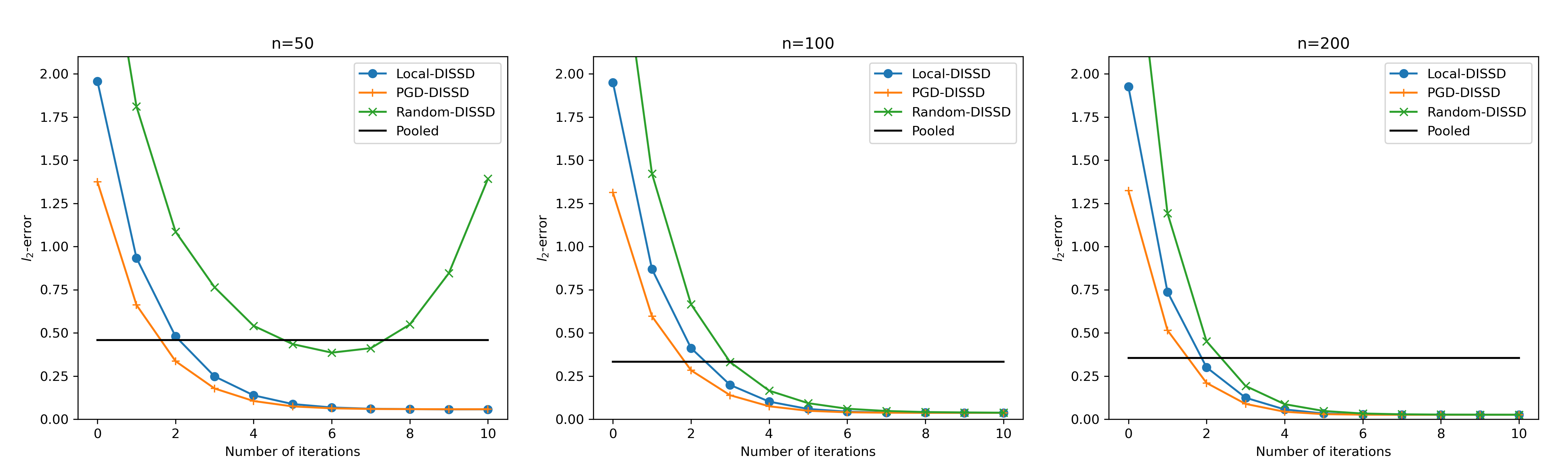}
	\end{center}
	\caption{The $\ell_2$-error over the number of iterations in Huber regression with different initializations. The local labeled sample size takes value in $\{50,100,200\}$, the local unlabeled sample size is $150$, the number of machine is $100$, and the dimension $p$ is $500$}
	\label{fig:huber_initial}
\end{figure}

\subsubsection{Median Regression} 

In the problem of median regression, we generate the noises $\epsilon$ from standard Cauchy distribution $\mathrm{Cauchy}(0, 1)$. We chose the kernel function $\mcK(\cdot)$ as biweight kernel function, which can be formulated as
\begin{equation*}
	\mcK(x)=
	\begin{cases}
		-\frac{315}{64}x^6+\frac{735}{64}x^4-\frac{525}{64}x^2+\frac{105}{64},\quad& \text{if }|x|\leq 1,\\
		0\quad&\text{if }|x|>1.
	\end{cases}
\end{equation*} 
 As for the bandwidth $h_{t}$, we uniformly take $h_t = 0.1$. The results are shown in Table \ref{tab:med_mnstar}, Table \ref{tab:med_ps} and Figure \ref{fig:med_iter}. In Figure \ref{fig:med_iter}, we exclude CSL from the comparison due to its original theory, which does not extend to non-smooth loss functions such as absolute deviation loss. Similar phenomena to the Huber regression models can be observed from the simulation results. In Table \ref{tab:med_mnstar}, we observe that the 1-Step DISSD method exhibits a decreasing $\ell_2$-error with the increase in the unlabeled sample size $n^*-n$, aligning with our theoretical findings in Corollary \ref{cor:betaT_conv}. However, an interesting trend emerges with the 5-step DISSD method when the number of machines $m$ is relatively small. In this scenario, the $\ell_2$-error when $n^*=550$ is slightly larger than that when $n^*=250$. We suspect that this phenomenon arises due to a larger unlabeled sample size introducing a greater bias in the smaller order terms.  

\begin{table}[H]
	\centering
	\small
	\caption{The $\ell_2$-errors and $F_1$-score and their standard errors (in parentheses) of one-step $\dissd$ and five-step $\dissd$,  under labeled local sample size $n=100$. The parameter dimension is $p=500$. Noises are generated from standard Cauchy distribution $\mathrm{Cauchy}(0,1)$ and the loss function is chosen as the absolute deviation loss.
	}\label{tab:med_mnstar}
	\bigskip
\begin{tabular}{c| c|cc| cc  }
\hline
\multirow{2}{*}{$m$}	&\multirow{2}{*}{$n^*$}	&\multicolumn{2}{c|}{ 1-Step DISSD }		&\multicolumn{2}{c}{ 5-Step DISSD }\\
	&	&$\ell_2$-error	&$F_1$-score	&$\ell_2$-error	&$F_1$-score\\ \hline
	&100   &1.357(0.077)	&0.687(0.009)	&3.598(15.238)	&0.685(0.013)	\\ 
20	    &250   &1.241(0.078)	&0.684(0.007)	&1.203(0.816)	&0.684(0.009)	\\
	&550   &1.215(0.115)	&0.681(0.005)	&1.505(0.931)	&0.679(0.005)	\\ \hline
	&100   &1.183(0.090)	&0.738(0.033)	&0.272(0.109)	&0.888(0.070)	\\ 
50	    &250   &0.966(0.085)	&0.721(0.025)	&0.148(0.061)	&0.895(0.055)	\\
	&550   &0.825(0.071)	&0.705(0.017)	&0.153(0.047)	&0.867(0.055)	\\ \hline
	&100   &1.144(0.085)	&0.859(0.053)	&0.201(0.072)	&0.994(0.020)	\\ 
100	    &250   &0.885(0.091)	&0.836(0.055)	&0.070(0.023)	&0.997(0.012)	\\
	&550   &0.671(0.091)	&0.786(0.045)	&0.052(0.015)	&0.996(0.013)\\ \hline
\end{tabular}
\end{table}

\begin{table}[H]
	\centering
	\small
	\caption{The $\ell_2$-errors and $F_1$-score and their standard errors (in parentheses) of one-step $\dissd$ and five-step $\dissd$,  under labeled local sample size $n=100$.The unlabeled sample size is $n^*-n=450$. Noises are generated from standard Cauchy distribution $\mathrm{Cauchy}(0,1)$ and the loss function is chosen as the absolute deviation loss.
	}\label{tab:med_ps}
	\bigskip
\begin{tabular}{c| c|cc| cc  }
\hline
\multirow{2}{*}{$p$}	&\multirow{2}{*}{$s$}	&\multicolumn{2}{c|}{ 1-Step DISSD }		&\multicolumn{2}{c}{ 5-Step DISSD }\\
	&	&$\ell_2$-error	&$F_1$-score	&$\ell_2$-error	&$F_1$-score\\ \hline
	&5	&0.449(0.094)	&0.998(0.016)	&0.033(0.012)	&1.000(0.000)	\\ 
200	    &10 &0.629(0.110)	&0.860(0.057)	&0.049(0.014)	&0.998(0.009)	\\
	&20	&0.875(0.150)	&0.721(0.007)	&0.142(0.021)	&0.834(0.028)	\\ \hline
	&5  &0.459(0.093)	&0.998(0.016)	&0.033(0.012)	&1.000(0.000)	\\ 
400	    &10 &0.651(0.104)	&0.799(0.051)	&0.052(0.016)	&0.998(0.010)	\\
	&20 &0.992(0.136)	&0.695(0.004)	&0.193(0.024)	&0.774(0.018)	\\ \hline
	&5  &0.471(0.090)	&0.996(0.020)	&0.035(0.012)	&1.000(0.000)	\\ 
600	    &10 &0.673(0.095)	&0.771(0.042)	&0.052(0.019)	&0.995(0.016)	\\
	&20	&1.069(0.095)	&0.687(0.002)	&0.248(0.022)	&0.740(0.011)\\ \hline
\end{tabular}
\end{table}

\begin{figure}[H]
	\begin{center}
		\includegraphics[width=1\textwidth]{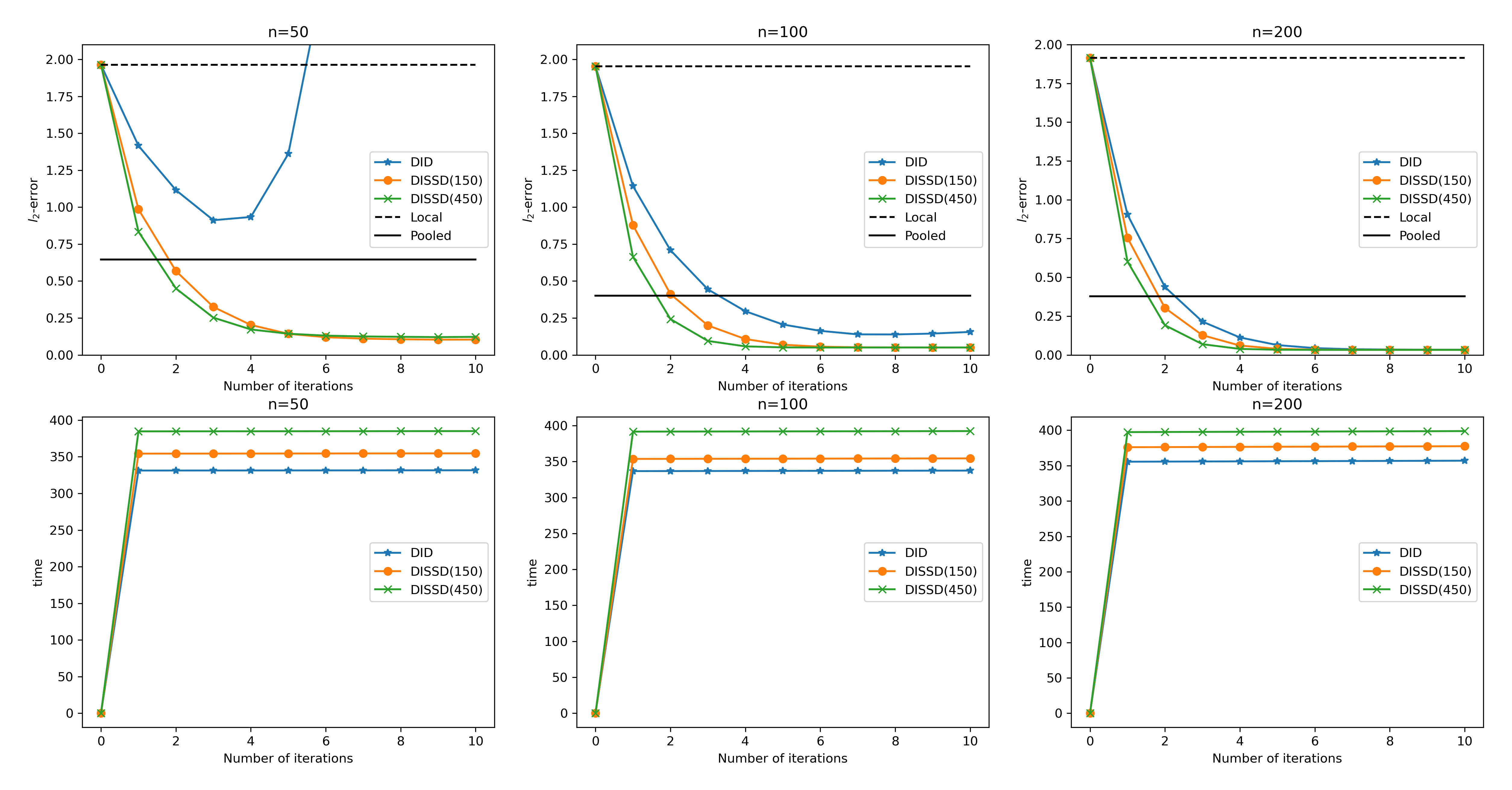}
	\end{center}
	\caption{The $\ell_2$-error (the first row) and the processing time (the second row) over the number of iterations in median regression. The local labeled sample size takes value in $\{50,100,200\}$, the local unlabeled sample size takes value in $\{0,150,450\}$, the number of machines is $100$, and the dimension $p$ is $500$.}
	\label{fig:med_iter}
\end{figure}

\subsubsection{Logistic Regression} 

In the logistic regression model, we generate the labels according to the distribution in \eqref{eq:logit_gen}. The detailed comparison of these approaches is given in Table \ref{tab:logit_mnstar}, Table \ref{tab:logit_ps}, and Figure \ref{fig:logit_iter}.

\begin{table}[h]
	\centering
	\small
	\caption{The $\ell_2$-errors and $F_1$-score and their standard errors (in parentheses) of one-step $\dissd$ and five-step $\dissd$,  under labeled local sample size $n=100$.The parameter dimension is $p=500$. Labels are generated according to \eqref{eq:logit_gen}, and the loss function is chosen as a logistic loss.
	}\label{tab:logit_mnstar}
	\bigskip
\begin{tabular}{c| c|cc| cc  }
\hline
\multirow{2}{*}{$m$}	&\multirow{2}{*}{$n^*$}	&\multicolumn{2}{c|}{ 1-Step DISSD }		&\multicolumn{2}{c}{ 5-Step DISSD }\\
	&	&$\ell_2$-error	&$F_1$-score	&$\ell_2$-error	&$F_1$-score\\ \hline
	&100	&1.416(0.054)	&0.713(0.011)	&0.896(0.300)            &0.694(0.007)	\\ 
20	    &250    &1.249(0.046)	&0.704(0.007)	&1.187(0.140)	        &0.683(0.003)	\\
	&550	&1.146(0.041)	&0.694(0.004)	&2.196(1.234)	    &0.676(0.003)	\\ \hline
	&100    &1.380(0.054)	&0.843(0.049)	&0.401(0.071)	&0.831(0.042)	\\ 
50	    &250    &1.180(0.044)	&0.842(0.041)	&0.349(0.046)	        &0.753(0.020)	\\
	&550    &1.025(0.040)	&0.792(0.028)	&0.515(0.056)	&0.710(0.008)	\\ \hline
	&100    &1.375(0.054)	&0.907(0.049)	&0.354(0.075)	&0.958(0.033)	\\ 
100	    &250    &1.170(0.041)	&0.961(0.035)	&0.162(0.039)	        &0.932(0.042)	\\
	&550	&1.005(0.037)	&0.958(0.035)	&0.189(0.035)	&0.835(0.037)\\ \hline
\end{tabular}
\end{table}

\begin{table}[h]
	\centering
	\small
	\caption{The $\ell_2$-errors and $F_1$-score and their standard errors (in parentheses) of one-step $\dissd$ and five-step $\dissd$,  under labeled local sample size $n=100$.The unlabeled sample size is $n^*-n=450$. Labels are generated according to \eqref{eq:logit_gen}, and the loss function is chosen as a logistic loss.
	}\label{tab:logit_ps}
	\bigskip
\begin{tabular}{c| c|cc| cc  }
\hline
\multirow{2}{*}{$p$}	&\multirow{2}{*}{$s$}	&\multicolumn{2}{c|}{ 1-Step DISSD }		&\multicolumn{2}{c}{ 5-Step DISSD }\\
	&	&$\ell_2$-error	&$F_1$-score	&$\ell_2$-error	&$F_1$-score\\ \hline
	&5	&0.652(0.042)	&0.986(0.037)	&0.057(0.024)	&1.000(0.000)	\\ 
200	    &10 &0.991(0.038)	&0.965(0.031)	&0.149(0.032)	&0.897(0.043)	\\
	&20	&1.559(0.032)	&0.796(0.018)	&0.386(0.034)	&0.731(0.008)	\\ \hline
	&5  &0.678(0.040)	&0.988(0.034)	&0.054(0.021)	&1.000(0.000)	\\ 
400	    &10 &1.010(0.036)	&0.961(0.033)	&0.171(0.031)	&0.856(0.036)	\\
	&20 &1.600(0.030)	&0.753(0.011)	&0.500(0.036)	&0.704(0.003)	\\ \hline
	&5  &0.680(0.040)	&0.984(0.039)	&0.052(0.021)	&1.000(0.000)	\\ 
600	    &10 &1.021(0.037)	&0.951(0.037)	&0.201(0.034)	&0.822(0.038)	\\
	&20	&1.604(0.031)	&0.727(0.009)	&0.617(0.038)	&0.693(0.002)\\ \hline
\end{tabular}
\end{table}

\begin{figure}[h]
	\begin{center}
		\includegraphics[width=1\textwidth]{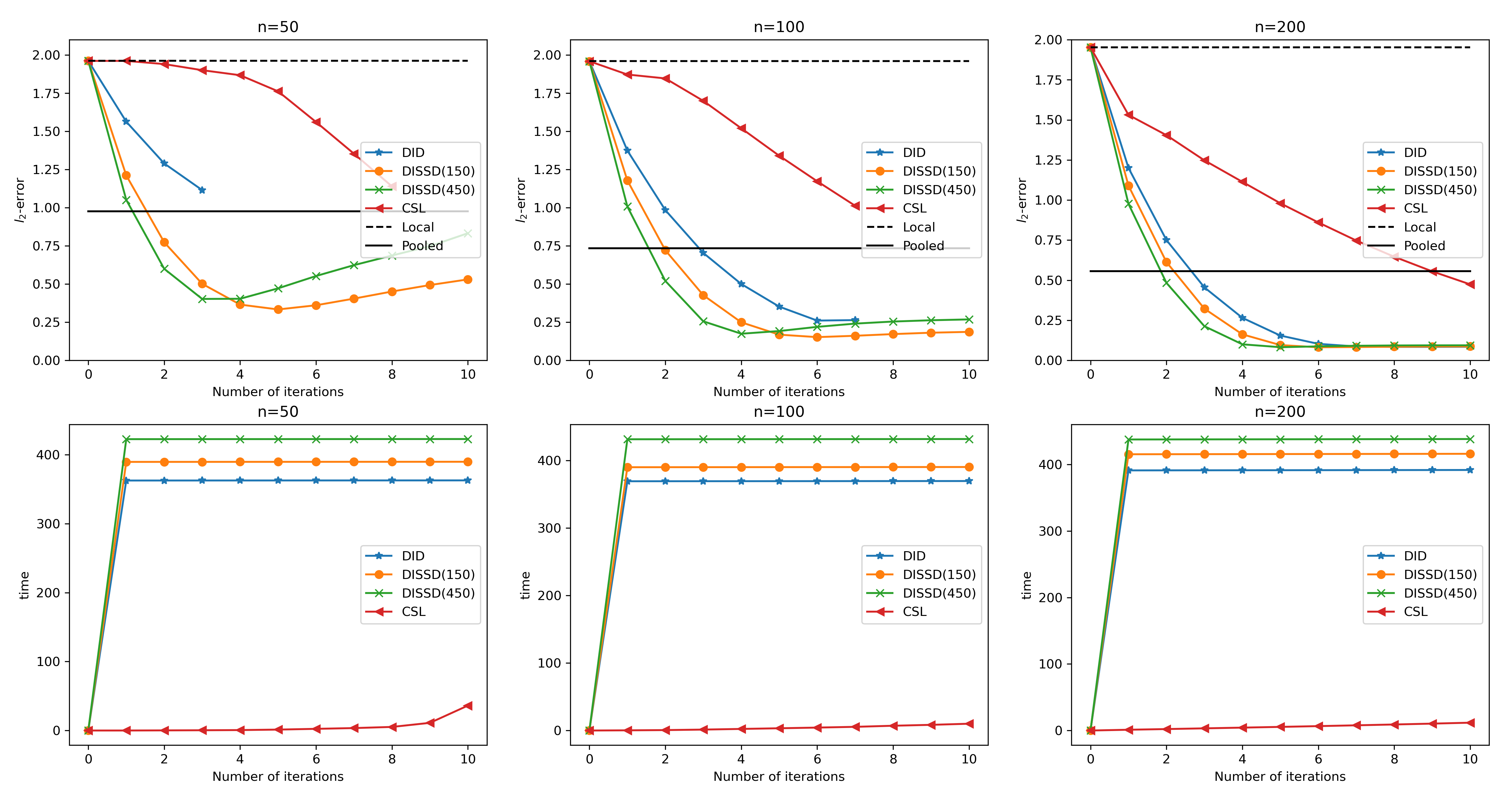}
	\end{center}
	\caption{The $\ell_2$-error (the first row) and the processing time (the second row) over the number of iterations in logistic regression. The local labeled sample size takes value in $\{50,100,200\}$, the local unlabeled sample size takes value in $\{0,150,450\}$, the number of machines is $100$, and the dimension $p$ is $500$.}
	\label{fig:logit_iter}
\end{figure}

We can observe similar trends as in the former simulation results, that the semi-supervised estimator always outperforms the fully supervised counterpart. In Table \ref{tab:logit_mnstar} and Figure \ref{fig:logit_iter}, we find that a larger unlabeled sample size plays a positive role in accelerating the convergence in the early-stage iterations, corroborating the findings in Corollary \ref{cor:glmT_conv}. However, an intriguing observation surfaces when either the local sample size $n$ or the number of machines $m$ is small: a larger unlabeled sample size adversely affects the statistical error in the post-stage iterations. This phenomenon challenges our current theoretical understanding and may necessitate the development of additional theoretical techniques for a comprehensive explanation. 

\subsection{Application to Real-World Benchmarks}

In the study, we analyze the Ames Housing data set\footnote{http://jse.amstat.org/v19n3/decock.pdf}, which was compiled by Dean De Cock for use in data science education. This data set contains all sales that occurred within Ames from 2006 to 2010. The response variable is the house sale price; the covariates include various house features. We aim to learn a sparse least square regression model to predict the house price by applying our proposed methods and to compare the performance in terms of prediction error. After weeding out the outliers and applying principal component analysis, we obtained 2902 observations and 51 features. We can compute that the $\ell_{0.1}$-sparsity of the empirical inverse covariance matrix of the covariate is $2.38$, which suffices Condition \ref{assump:X}.  We randomly partition the data set into 2000 training data and 612 testing data. Among them, we remove the label of 290 data and treat it as the unlabeled covariates. We perform 100 random partitions of the dataset and report the averaged prediction error and processing time on the testing set. We consider three cases where $(m,n)=(200,10),(100,20)$ and $(40,50)$ respectively. In our experimental analysis, we conducted a comparative evaluation of our DISSD method against the CSL method from \cite{jordan_etal.2019}, the CEASE method from \cite{fan_guo_etal.2019}, and the Debias-DC method detailed in \cite{battey2018distributed}. The results are summarized in Figure \ref{fig:realdata_iter}. Notably, the semi-supervised methods consistently outperform their fully-supervised counterparts. Particularly for small sample sizes (specifically $n=10$ and $20$), CSL, CEASE, and Debias-DC all exhibit instability, leading to divergence. Even at a sample size of $50$, these methods are surpassed by our DISSD method in terms of performance. From the computational perspective, both CEASE and Debias-DC demonstrate considerably longer processing times compared to other methods.

\begin{figure}[h]
	\begin{center}
		\includegraphics[width=1\textwidth]{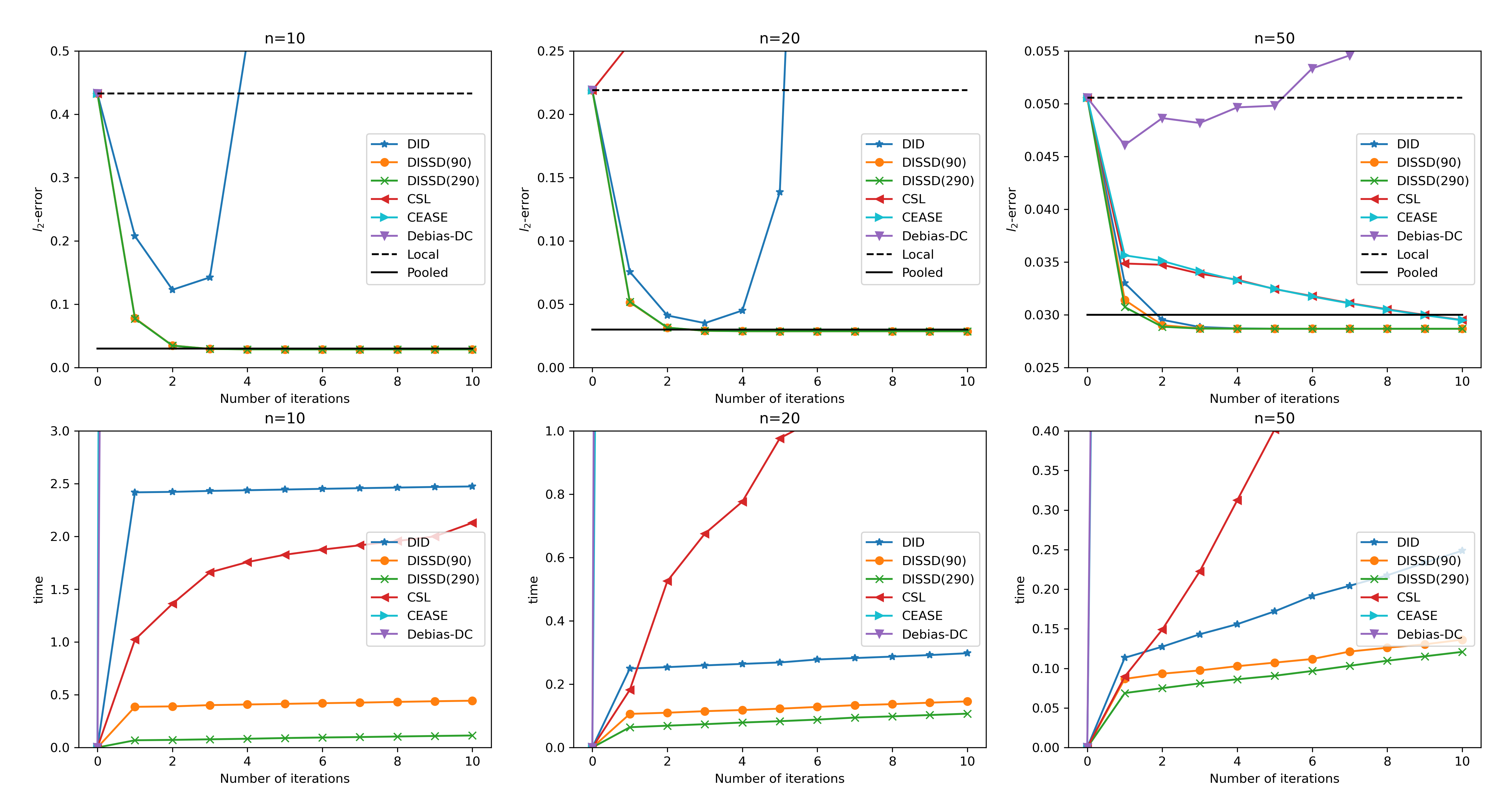}
	\end{center}
	\caption{The prediction error (the first row) and the processing time (the second row) over the number of iterations in linear regression for Ames Housing dataset. The total labeled training sample size is $2000$, the local labeled sample size varies in $\{10,20,50\}$, and the local unlabeled sample size is $\{0,90,290\}$, and the dimension $p$ is $51$.}
	\label{fig:realdata_iter}
\end{figure}


\section{Concluding Remarks and Future Study}\label{sec:conclude}

In this paper, we investigate the problem of semi-supervised sparse learning in a distributed setup and demonstrate that incorporating unlabeled data can enhance the statistical accuracy in distributed learning. Throughout this paper, we construct the inverse covariance matrix $\vect{\Omega}$ using SCIO \cite{liu_luo.2015}, which requires $\vect{\Omega}$ to be sparse in each column. However, there exist other debiasing approaches, such as \cite{javanmard_montanari.2014jmlr, geer_buhlmann_etal.2014, zhang_zhang.2014, cai_cai_guo.2021jrssb}, that can eliminate the sparsity constraint on $\vect{\Omega}$. It would be interesting to explore the application of these approaches to our DISSD method in future work to see if similar results can be achieved. Furthermore, in addition to debiased estimators, other methods for sparse learning exist, such as adding $\ell_1$ penalty \citep{tibshirani1996regression} or non-convex penalties \citep{fan_li.2001, zhang2010jmlr, zhang.2010aos}, or using the iterative hard thresholding method \citep{blumensath_davies.2009, jain_tewari_kar.2014nips, yuan_li_zhang.2018jmlr}. Investigating the use of distributed semi-supervised methods for these techniques is of significant importance.


\bibliographystyle{asa}
\bibliography{dist_semi_debias}

\newpage


\section{Appendix}

The appendix consists of three parts. Appendix \ref{sec:lemma} includes some technical lemmas for the proof. The proof of theorems in Section \ref{sec:mest_theory} in the main paper are given in Appendix \ref{sec:mest_proof}. The proof of theories in Section \ref{sec:glm_theory} is presented in Appendix \ref{sec:glm_proof}. 


\subsection{Technical Lemmas}	\label{sec:lemma}

\begin{lemma}	\label{lemma:cai_liu}
	(Exponential Inequality, Lemma 1 in \cite{cai_liu.2011}) Let $X_1,...,X_n$ be independent random variables all with means zero. Suppose that there exist some constants $\eta>0$ and $B_{n}$ such that $\sum_{i=1}^{n}\mfE (X_{i},\eta)\leq B_{n}^{2}$, where $\mfE(X,\eta)=\mbE[X^2\exp(\eta |X|)]$. Then uniformly for $0<x\leq B_{n}$ and $n\geq 1$, there is
	\begin{equation*}
		\mbP\left\{\sum_{i=1}^nX_i\geq (\eta+\eta^{-1})B_{n}x\right\}\leq\exp\left(-x^2\right).
	\end{equation*}
\end{lemma}

\begin{lemma}	\label{lem:bound_precision}
	(Theorem 1 of \cite{wu_wang_liu.2023aism}) Under Condition \ref{assump:X}, the estimator $\hat{\vect{\Omega}}_{\mcH_1}$ of the inverse covariance matrix $\vect{\Omega}$ satisfies
	\begin{equation*}
		\|\hat{\vect{\Omega}}_{\mcH_1}-\vect{\Omega}\|_{L_1} =O_{\mbP}\left( s\Big(\frac{\log p}{n^*}\Big)^{(1-q)/2}\right).
	\end{equation*}
\end{lemma}

\subsection{Proof of Theories in Section \ref{sec:mest_theory}}	\label{sec:mest_proof}

\begin{lemma}	\label{lem:bound_H0}
	Under Condition \ref{assump:X} to \ref{assump:kernel}, there is
	\begin{equation}	\label{eq:H_bound_nonsmooth}
		|\hat{H}^{(0)}-h'(0)|=O_{\mbP}\left(\sqrt{\frac{s\log p}{mnh_1}}+r_n\right).
	\end{equation}
\end{lemma}

\begin{proof}
	For simplicity, we assume that $f'(\cdot)$ only has one discontinuous point $x_1$, \ie, $K=1$. Define $\mcS=\big\{S_q\subseteq\{1,...,p\}\big| |S_q|_0=C_rs \big\}$ as the collection of index sets having $C_rs$ elements. Then we know that $|\mcS|=C_p^{C_rs}\leq p^{C_rs}$. For each $S_q\in\mcS$, we denote the parameter sets $\Theta_q=\{\vect{\beta}: \supp(\vect{\beta})\subseteq S_q,\; |\vect{\beta}-\vect{\beta}^*|_2\leq r_n\}$, and construct the following random variable
	\begin{equation*}
		D_{h}(\vect{\beta})=\frac{1}{mn}\sum_{j=1}^m\sum_{i\in\mcH_j}f''\big(Y_i-\vect{X}^{\tp}_{i}\vect{\beta}\big)+\frac{\Delta_1}{mnh_1}\sum_{j=1}^m\sum_{i\in\mcH_j}\mcK\Big(\frac{Y_i-\vect{X}_i^{\tp}\vect{\beta}-x_1}{h_1}\Big).
	\end{equation*} 
	We construct $\mfN_q$, an $r_nn^{-M}$-net of the set $\Theta_q$. By definition we know that $\Theta_q$ is a ball in $\mbR^{C_rs}$ with radius $r_n$. By Lemma 5.2 of \cite{vershynin.2010}, we have $\card(\mfN_q)\leq (1+2n^M)^s$. From Lipschitz continuity of $\mcK'(\cdot)$ in Assumption \ref{assump:kernel}, we have
	\begin{align*}
		&\max_{\tilde{\vect{\beta}}\in\mfN_q}\sup_{\tilde{\vect{\beta}}:|\tilde{\vect{\beta}}-\vect{\beta}|_2\leq r_nn^{-M}}\left|D_{h}(\vect{\beta})-D_{h}(\tilde{\vect{\beta}})\right|\\
		\leq& \max_{\tilde{\vect{\beta}}\in\mfN_q}\sup_{\vect{\beta}:|\tilde{\vect{\beta}}-\vect{\beta}|_2\leq r_nn^{-M}}\frac{1}{mn}\sum_{j=1}^m\sum_{i\in\mcH_j}\Big|f''\big(Y_i-\vect{X}^{\tp}_{i}\vect{\beta}\big)-f''\big(Y_i-\vect{X}^{\tp}_{i}\tilde{\vect{\beta}}\big)\Big|\\
		&+\frac{C_kr_n}{mn^{M+1}h_1^2}\sum_{j=1}^m\sum_{i\in\mcH_j}|\vect{X}_{i,S_q}|_2.
	\end{align*}
	From Assumption \ref{assump:X}, we can show that
	\begin{equation*}
		\frac{C_kr_n}{n^{M+1}h_1^2}\sum_{j=1}^m\sum_{i\in\mcH_j}|\vect{X}_{i,S_k}|_2=O_{\mbP}(n^{-2}),
	\end{equation*}
	by letting $M$ large enough. Moreover, from Assumption \ref{assump:fpp_mfE}, we can apply Lemma \ref{lemma:cai_liu} to the i.i.d. random variables
	\begin{equation*}
		\sup_{\vect{\beta}:|\tilde{\vect{\beta}}-\vect{\beta}|_2\leq r_nn^{-M}}\Big|f''\big(Y_i-\vect{X}^{\tp}_{i}\vect{\beta}\big)-f''\big(Y_i-\vect{X}^{\tp}_{i}\tilde{\vect{\beta}}\big)\Big|,
	\end{equation*}
	and prove that, for any $\gamma>0$, there exists a constant $C_1$ such that
	\begin{equation*}
		\mbP\left\{\frac{1}{mn}\sum_{j=1}^m\sum_{i\in\mcH_j}\sup_{\vect{\beta}:|\tilde{\vect{\beta}}-\vect{\beta}|_2\leq r_nn^{-M}}\Big|f''\big(Y_i-\vect{X}^{\tp}_{i}\vect{\beta}\big)-f''\big(Y_i-\vect{X}^{\tp}_{i}\tilde{\vect{\beta}}\big)\Big|\geq C_1\frac{s\log p}{mn}\right\}=O(p^{-\gamma s}).
	\end{equation*}
	Therefore we have
	\begin{equation}	\label{eq:hath_lemma_term1}
	\begin{aligned}
		\max_{\tilde{\vect{\beta}}\in\mfN_q}\sup_{\tilde{\vect{\beta}}:|\tilde{\vect{\beta}}-\vect{\beta}|_2\leq r_nn^{-M}}\left|D_{h_1}(\vect{\beta})-D_{h_1}(\tilde{\vect{\beta}})\right|=O_{\mbP}\Big(\frac{s\log p}{mn}\Big).
	\end{aligned}
	\end{equation}
	Moreover, by Lipschitz continuity of $\mbE\{f''(\cdot)\}$ and $\rho(\cdot)$ in Assumption \ref{assump:fpp_mfE} and \ref{assump:lip_pdf},  
	for every $\tilde{\vect{\beta}}\in\mfN_0$, we can compute that
	\begin{align*}
		\mbE\{D_{h_1}(\tilde{\vect{\beta}})\}=&\frac{1}{mn}\mbE\Big[\sum_{i\in\mcH_j}f''\big(Y_i-\vect{X}^{\tp}_{i}\tilde{\vect{\beta}}\big)+\int_{-\infty}^{\infty}\mcK(x)\sum_{j=1}^m\sum_{i\in\mcH_j}\rho\{x_1+h_1x+\vect{X}^{\tp}_i(\tilde{\vect{\beta}}-\vect{\beta}^*)\}\diff x\Big]\\
		=&\mbE\{f''(Y-\vect{X}^{\tp}\vect{\beta}^*)\}+\rho(x_1)+O\Big(h_1+\frac{1}{mn}\sum_{j=1}^m\sum_{i\in\mcH_j}\mbE\big|\vect{X}_i^{\tp}(\tilde{\vect{\beta}}-\vect{\beta}^*)\big|\Big).
	\end{align*}
	From Assumption \ref{assump:X}, we have
	\begin{equation}	\label{eq:hath_lemma_term2}
		\sup_{\tilde{\vect{\beta}}\in\mfN_0}\left|\mbE\{D_{h_1}(\tilde{\vect{\beta}})\}-h'(0)\right|=O(h_1+r_n).
	\end{equation}
	On the other hand, we define
	\begin{equation*}
		\xi_i(\vect{\beta})=f''(Y_i-\vect{X}^{\tp}_i\vect{\beta})+\frac{1}{h_1}\mcK\left(\frac{Y_i-\vect{X}^{\tp}_{i}\vect{\beta}-x_1}{h_1}\right).
	\end{equation*}
	Compute that
	\begin{align*}
		\mbE\big|f''(Y_i-\vect{X}_i^{\tp}\vect{\beta})\big|^2\asymp1,&\\
		\mbE\Big\{\frac{1}{h_1^2}\mcK^2\Big(\frac{Y_i-\vect{X}^{\tp}_{i}\vect{\beta}-x_1}{h_1}\Big)\Big\}&=\frac{1}{h_1}\mbE\Big[\int_{-\infty}^{\infty}\{\mcK(x)\}^2f(x_1+h_1x+\vect{X}_i^{\tp}(\vect{\beta}-\vect{\beta}^*))\diff x\Big]\asymp \frac{1}{h_1}.
	\end{align*}
	Therefore we have $\mbE\big|\xi_i^2(\vect{\beta})\big|\asymp\frac{1}{h_1}$.
	
	Applying Lemma \ref{lemma:cai_liu}, for any $\gamma>0$, we have that there exists a constant $C_2$ such that
	\begin{align*}
		\sup_{\tilde{\vect{\beta}}\in\mfN_q}\mbP\left[\Big|D_{h_1}(\tilde{\vect{\beta}})-\mbE\big\{D_{h_1}(\tilde{\vect{\beta}})\big\}\Big|\geq C_2\sqrt{\frac{s\log p}{h_1mn}}\right]=O(p^{-\gamma s}).
	\end{align*}
	By letting $\gamma>M$, together with \eqref{eq:hath_lemma_term1} and \eqref{eq:hath_lemma_term2}, we have that
	\begin{equation}	\label{eq:hath_lemma_term3}
		\sup_{\vect{\beta}\in\Theta}\left|D_{h_1}(\vect{\beta})-H(0)\right|=O_{\mbP}\Big(\sqrt{\frac{s\log p}{mnh_1}}+r_n\Big).
	\end{equation}
	which proves the desired result. For the case of $K>1$, one can obtain the same rate since the additional $\mcK(\cdot)$ terms can be discussed by the same argument.
\end{proof}

It is worthwhile noting that, when the loss function $\mcL(x)$ is smooth, there is no $\mcK(\cdot)$ term in $\hat{H}^{(0)}$. Follow the same procedure we can prove that
\begin{equation}	\label{eq:hath_smooth}
	\big|\hat{H}^{(0)}-h'(0)\big|=O_{\mbP}\Big(\sqrt{\frac{s\log n}{mn}}+r_n\Big).
\end{equation}

\begin{lemma}	\label{lem:bound_Ub}
	Assume Condition \ref{assump:X} and \ref{assump:fpp_mfE} Hold. Define
	\begin{equation*}
		\vect{U}(\hat{\vect{\beta}}^{(0)}) = \frac{1}{m}\sum_{j=1}^m\nabla\mcL_j(\hat{\vect{\beta}}^{(0)} ) - \frac{1}{m}\sum_{j=1}^m\nabla\mcL_j(\vect{\beta}^{*} )-\mbE\Big[\vect{X}f'(\vect{X}^{\tp}\hat{\vect{\beta}}^{(0)}-Y)-\vect{X}f'(\vect{X}^{\tp}\vect{\beta}^{*}-Y)\Big].
	\end{equation*} 
	\begin{itemize}
		\item[(i)] If Condition \ref{assump:smooth} holds, then there is
			\begin{equation*}
				|\vect{U}(\hat{\vect{\beta}}^{(0)})|_{\infty} = O_{\mbP}\left(r_n\sqrt{\frac{s\log p}{mn}}\right).
			\end{equation*}
		\item[(ii)] If Condition \ref{assump:non-smooth} holds, then there is
			\begin{equation*}
				|\vect{U}(\hat{\vect{\beta}}^{(0)})|_{\infty} = O_{\mbP}\left(\sqrt{\frac{r_ns\log p}{mn}}\right).
			\end{equation*}
	\end{itemize}
\end{lemma}
	
\begin{proof}
	Firstly we consider the case that Condition \ref{assump:non-smooth} holds true. Define $\mcS=\big\{S_q\subseteq\{1,...,p\}\big| |S_q|_0=C_rs \big\}$ as the collection of index sets having $C_rs$ elements. Then we know that $|\mcS|=C_p^{C_rs}\leq p^{C_rs}$. For each $S_q\in\mcS$, we denote the parameter sets $\Theta_q=\{\vect{\beta}: \supp(\vect{\beta})\subseteq S_q,\; |\vect{\beta}-\vect{\beta}^*|_2\leq r_n\}$. For each $q$, we construct $\mfN_q$, an $r_nn^{-M}$-net of the set $\Theta_q$. By definition we know that $\Theta_q$ is a ball in $\mbR^{C_rs}$ with radius $r_n$. By Lemma 5.2 of \cite{vershynin.2010}, we have $\card(\mfN_q)\leq (1+2n^M)^s$. Then we have
	\begin{align*}
		&\sup_{\vect{\beta}\in\Theta_q}\Big|\vect{U}(\vect{\beta})\Big|_{\infty}-\sup_{\tilde{\vect{\beta}}\in\mfN_q}\Big|\vect{U}(\tilde{\vect{\beta}})\Big|_{\infty}\\
		\leq&\max_{\tilde{\vect{\beta}}\in\mfN_q}\sup_{\tilde{\vect{\beta}}:|\vect{\beta}-\tilde{\vect{\beta}}|_2\leq r_nn^{-M}}\Big|\vect{U}(\vect{\beta})-\vect{U}(\tilde{\vect{\beta}})\Big|_{\infty}\\
		\leq&\max_{1\leq l\leq p}\max_{\tilde{\vect{\beta}}\in\mfN_q}\frac{1}{mn}\sum_{i=1}^m\sup_{\tilde{\vect{\beta}}:|\vect{\beta}-\tilde{\vect{\beta}}|_2\leq r_nn^{-M}}\Big|X_{i,l}f'(Y_i-\vect{X}_i^{\tp}\vect{\beta})-\mbE\big\{X_{i,l}f'(Y_i-\vect{X}_i^{\tp}\vect{\beta})\big\}\\
		&-X_{i,l}f'(Y_i-\vect{X}_{i}^{\tp}\tilde{\vect{\beta}})+\mbE\big\{X_{i,l}f'(Y_i-\vect{X}_i^{\tp}\tilde{\vect{\beta}})\big\}\Big|.
	\end{align*}
	For simplicity we denote
	\begin{align*}
		V_{i,l}\big(\tilde{\vect{\beta}}\big)=&\sup_{\vect{\beta}:|\vect{\beta}-\tilde{\vect{\beta}}|_2\leq r_nn^{-M}}\big|X_{i,l}f'(Y_i-\vect{X}_i^{\tp}\vect{\beta})-\mbE\big\{X_{i,l}f'(Y_i-\vect{X}_i^{\tp}\vect{\beta})\big\}\\
		&-X_{i,l}f'(Y_i-\vect{X}_{i}^{\tp}\tilde{\vect{\beta}})+\mbE\big\{X_{i,l}f'(Y_i-\vect{X}_i^{\tp}\tilde{\vect{\beta}})\big\}\big|.
	\end{align*}
	From Condition \ref{assump:non-smooth}, we have
	\begin{align*}
		&\sup_{1\leq l\leq p}\sup_{\tilde{\vect{\beta}}\in\mfN_q}\mbE\{V_{i,l}\big(\tilde{\vect{\beta}}\big)\}=o((mn)^{-1}),\\
		&\sup_{1\leq l\leq p}\sup_{\tilde{\vect{\beta}}\in\mfN_q}\mfE\big[V_{i,l}\big(\tilde{\vect{\beta}}\big)-\mbE\{V_{i,l}\big(\tilde{\vect{\beta}}\big)\},\eta_2\big]=o((mn)^{-1}).
	\end{align*}
	Therefore by using Lemma \ref{lemma:cai_liu}, we can show that, for every $\gamma>1$, there exists $C$ large enough such that
	\begin{align*}
		\sup_{1\leq l\leq p}\sup_{\tilde{\vect{\beta}}\in\mfN_q}\mbP\Big[\frac{1}{n}\sum_{i\in\mcH_j}V_{i,l}\big(\tilde{\vect{\beta}}\big)\geq C\frac{s\log p}{mn}\Big]=O(p^{-\gamma s}).
	\end{align*}
	By letting $\gamma>M+1$, we have
	\begin{equation}	\label{eq:ubound_term1}
		\sup_{\vect{\beta}\in\Theta_q}\Big|\vect{U}(\vect{\beta})\Big|_{\infty}-\sup_{\tilde{\vect{\beta}}\in\mfN_q}\Big|\vect{U}(\tilde{\vect{\beta}})\Big|_{\infty}=O_{\mbP}\Big(\frac{s\log p}{mn}\Big).
	\end{equation} 
	From Condition \ref{assump:non-smooth}, we can show that
	\begin{align*}
		\sup_{\tilde{\vect{\beta}}\in\mfN_0}\mfE\Big[&X_{i,l}\mcL'(Y_i-\vect{X}_i^{\tp}\tilde{\vect{\beta}})-\mbE\big\{X_{i,l}\mcL'(Y_i-\vect{X}_i^{\tp}\tilde{\vect{\beta}})\big\}\\
		&-X_{i,l}\mcL'(Y_i-\vect{X}_i^{\tp}\vect{\beta}^*)+\mbE\big\{X_{i,l}\mcL'(Y_i-\vect{X}_i^{\tp}\vect{\beta}^*)\big\},\eta_2\Big]=O_{\mbP}(r_n).
	\end{align*}
	Then again by Lemma \ref{lemma:cai_liu} we have that, for every $\gamma>1$, there exists $C_1$ large enough such that
	\begin{equation*}
		\sup_{1\leq l\leq p}\sup_{\tilde{\vect{\beta}}\in\mfN_q}\mbP\Big\{\big|U_{l}(\tilde{\vect{\beta}})\big|\geq C_1\sqrt{\frac{r_ns\log p}{mn}}\Big\}=O_{\mbP}(p^{-\gamma s}).
	\end{equation*}
	Therefore we have
	\begin{equation}	\label{eq:ubound_term2}
		|\vect{U}(\hat{\vect{\beta}}^{(0)})|_{\infty}\leq\max_q\sup_{\tilde{\vect{\beta}}\in\mfN_q}\Big|\vect{U}(\tilde{\vect{\beta}})\Big|_{\infty}=O_{\mbP}\Big(\sqrt{\frac{r_ns\log p}{mn}}\Big).
	\end{equation}
	Combining \eqref{eq:ubound_term1} and \eqref{eq:ubound_term2}, the second part of the lemma is proved. Follow the same idea we can prove the first part. Indeed from Condition \ref{assump:smooth} we can show that
	\begin{equation*}
		\sup_{\vect{\beta}\in\Theta_0}\Big|\vect{U}(\vect{\beta})\Big|_{\infty}-\sup_{\tilde{\vect{\beta}}\in\mfN_0}\Big|\vect{U}(\tilde{\vect{\beta}})\Big|_{\infty}=O_{\mbP}\big((mn)^{-1}\big),\quad \text{and }\quad\sup_{\tilde{\vect{\beta}}\in\mfN_0}\Big|\vect{U}_j(\tilde{\vect{\beta}})\Big|_{\infty}=O_{\mbP}\Big(r_n\sqrt{\frac{s\log p}{mn}}\Big),
	\end{equation*}
	and then the first part can be proved.
\end{proof}

\begin{proof}[Proof of Theorem \ref{thm:sm_conv}, \ref{thm:nsm_conv} and Corollary \ref{cor:betaT_conv}]
	From \eqref{eq:dss_mest}, we have the following decomposition.
	\begin{align*}	
	\bar{\vect{\beta}}^{(1)}-\vect{\beta}^{*}=&\hat{\vect{\beta}}^{(0)} - \frac{1}{\hat{H}(0)}\hat{\vect{\Omega}}_{\mcH_1}\frac{1}{m}\sum_{j=1}^m\nabla\mcL_j(\hat{\vect{\beta}}^{(0)} )-\vect{\beta}^*\stepcounter{equation}\tag{\theequation}\label{eq:dist_debias_decomp}\\
		=&\hat{\vect{\beta}}^{(0)}-\vect{\beta}^* - \frac{1}{\hat{H}(0)}\hat{\vect{\Omega}}_{\mcH_1}\left\{ \mbE\Big[\vect{X}f'(\vect{X}^{\tp}\hat{\vect{\beta}}^{(0)}-Y)-\vect{X}f'(\vect{X}^{\tp}\vect{\beta}^{*}-Y)\Big]+\frac{1}{m}\sum_{j=1}^m\nabla\mcL_j(\vect{\beta}^{*}) +\vect{U}(\hat{\vect{\beta}}^{(0)}) \right\} \\
		=&\hat{\vect{\beta}}^{(0)}-\vect{\beta}^* - \frac{1}{\hat{H}(0)}\hat{\vect{\Omega}}_{\mcH_1}\left\{ H(0)\vect{\Sigma}(\hat{\vect{\beta}}^{(0)}-\vect{\beta}^*)+O(r_n^2)+ \frac{1}{m}\sum_{j=1}^m\nabla\mcL_j(\vect{\beta}^{*}) +\vect{U}(\hat{\vect{\beta}}^{(0)}) \right\} \\
		=&\Big\{\frac{1}{H(0)}\vect{\Omega}-\frac{1}{\hat{H}(0)}\hat{\vect{\Omega}}_{\mcH_1}\Big\} H(0)\vect{\Sigma}(\hat{\vect{\beta}}^{(0)}-\vect{\beta}^*)- \frac{1}{\hat{H}(0)}\hat{\vect{\Omega}}_{\mcH_1}\left\{\frac{1}{m}\sum_{j=1}^m\nabla\mcL_j(\vect{\beta}^{*}) +\vect{U}(\hat{\vect{\beta}}^{(0)})+O(r_n^2) \right\},
	\end{align*}
	where the third line follows from Condition \ref{assump:fpp_mfE} and
	\begin{equation*}
		\vect{U}(\hat{\vect{\beta}}^{(0)}) = \frac{1}{m}\sum_{j=1}^m\nabla\mcL_j(\hat{\vect{\beta}}^{(0)} ) - \frac{1}{m}\sum_{j=1}^m\nabla\mcL_j(\vect{\beta}^{*} )-\mbE\Big[\vect{X}f'(\vect{X}^{\tp}\hat{\vect{\beta}}^{(0)}-Y)-\vect{X}f'(\vect{X}^{\tp}\vect{\beta}^{*}-Y)\Big].
	\end{equation*}
	Taking norms on both sides we have that
	\begin{align*}
		\big|\bar{\vect{\beta}}^{(1)}-\vect{\beta}^{*}\big|_{\infty}\leq& H(0)\Big|\frac{1}{H(0)}-\frac{1}{\hat{H}(0)}\Big|\times \big|\hat{\vect{\beta}}^{(0)}-\vect{\beta}^*\big|_{\infty} + \frac{H(0)}{\hat{H}(0)}\Big\|\vect{\Omega}-\hat{\vect{\Omega}}_{\mcH_1}\Big\|_{L_1} \Big|\vect{\Sigma}(\hat{\vect{\beta}}^{(0)}-\vect{\beta}^*)\Big|_{\infty}\\
		&+ \Big\|\frac{1}{\hat{H}(0)}\hat{\vect{\Omega}}_{\mcH_1}\Big\|_{L_1}\left\{\Big|\frac{1}{m}\sum_{j=1}^m\nabla\mcL_j(\vect{\beta}^{*})\Big|_{\infty} +\Big|\vect{U}(\hat{\vect{\beta}}^{(0)})\Big|_{\infty}+O(r_n^2) \right\}.
	\end{align*}
	Under Condition \ref{assump:smooth}, we applying Lemma \ref{lem:bound_precision}, \ref{lem:bound_H0} and \ref{lem:bound_Ub}, and know that
	\begin{align*}
		\big|\bar{\vect{\beta}}^{(1)}-\vect{\beta}^{*}\big|_{\infty}=&O_{\mbP}\Big(r_n\sqrt{\frac{s\log p}{mn}}+r_n^2\Big) + O_{\mbP}\Big(r_ns\Big(\frac{\log p}{n^*}\Big)^{(1-q)/2}\Big) + O_{\mbP}\Big(\sqrt{\frac{\log p}{mn}}+r_n\sqrt{\frac{s\log p}{mn}}+r_n^2\Big)\\
		=&O_{\mbP}\left(\sqrt{\frac{\log p}{mn}}+r_n\sqrt{\frac{s\log p}{mn}}+r_ns\Big(\frac{\log p}{n^*}\Big)^{(1-q)/2}+r_n^2\right).
	\end{align*}
	When the threshold level $\tau$ is larger than $\big|\bar{\vect{\beta}}^{(1)}-\vect{\beta}^{*}\big|_{\infty}$, we know that $|\bar{\beta}^{(1)}_l|<\tau$ for $l\notin S$, therefore $\hat{\beta}_l^{(1)}=0$. For $l\in S$, there is 
	\begin{equation*}
		\big|\hat{\beta}^{(1)}_l-\beta_l^{*}\big|\leq \big|\bar{\beta}^{(1)}_l-\beta_l^{*}\big|+\tau = O_{\mbP}\left(\sqrt{\frac{\log p}{mn}}+r_n\sqrt{\frac{s\log p}{mn}}+r_ns\Big(\frac{\log p}{n^*}\Big)^{(1-q)/2}+r_n^2\right).
	\end{equation*}
	Therefore, 
	\begin{align*}
		|\hat{\vect{\beta}}^{(1)}-\vect{\beta}^*|_1 =& |\hat{\vect{\beta}}^{(1)}_S-\vect{\beta}^*_S|_1 + |\hat{\vect{\beta}}^{(1)}_{S^c}-\vect{\beta}^*_{S^c}|_1\leq s |\hat{\vect{\beta}}^{(1)}_S-\vect{\beta}^*_S|_{\infty}\\
		=&O_{\mbP}\left(s\sqrt{\frac{\log p}{mn}}+r_ns\sqrt{\frac{s\log p}{mn}}+r_ns^2\Big(\frac{\log p}{n^*}\Big)^{(1-q)/2}+sr_n^2\right).
	\end{align*} 
	Similarly, we can prove that
	\begin{equation*}	
	\begin{aligned}
		\big|\hat{\vect{\beta}}^{(1)}-\vect{\beta}^*\big|_{2}=&O_{\mbP}\left(\sqrt{\frac{s\log p}{mn}}+r_n\sqrt{\frac{s^2\log p}{mn}}+r_ns^{3/2}\Big(\frac{\log p}{n^*}\Big)^{(1-q)/2}+\sqrt{s}r_n^2\right),
	\end{aligned}
	\end{equation*}
	which proves Theorem \ref{thm:sm_conv}. Under the non-smooth condition \ref{assump:non-smooth}, the rate becomes
	\begin{align*}
		\big|\bar{\vect{\beta}}^{(1)}-\vect{\beta}^{*}\big|_{\infty}=&O_{\mbP}\Big(r_n\sqrt{\frac{s\log p}{r_n mn}}+r_n^2\Big) + O_{\mbP}\Big(r_ns\Big(\frac{\log p}{n^*}\Big)^{(1-q)/2}\Big) + O_{\mbP}\Big(\sqrt{\frac{\log p}{mn}}+\sqrt{\frac{r_ns\log p}{mn}}+r_n^2\Big)\\
		=&O_{\mbP}\left(\sqrt{\frac{\log p}{mn}}+\sqrt{\frac{r_ns\log p}{mn}}+r_ns\Big(\frac{\log p}{n^*}\Big)^{(1-q)/2}+r_n^2\right).
	\end{align*}
	Following the same argument as above, we can show that
	\begin{equation*}	
	\begin{aligned}
		\big|\hat{\vect{\beta}}^{(1)}-\vect{\beta}^*\big|_{\infty}=&O_{\mbP}\left(\sqrt{\frac{\log p}{mn}} + \sqrt{\frac{r_ns\log p}{mn}}+r_ns\Big(\frac{\log p}{n^*}\Big)^{(1-q)/2}+r_n^2\right),\\
		\big|\hat{\vect{\beta}}^{(1)}-\vect{\beta}^*\big|_{2}=&O_{\mbP}\left(\sqrt{\frac{s\log p}{mn}}+\sqrt{\frac{r_ns^2\log p}{mn}}+r_ns^{3/2}\Big(\frac{\log p}{n^*}\Big)^{(1-q)/2}+\sqrt{s}r_n^2\right),\\
		\big|\hat{\vect{\beta}}^{(1)}-\vect{\beta}^*\big|_{1}=&O_{\mbP}\left(s\sqrt{\frac{\log p}{mn}}+s\sqrt{\frac{r_ns\log p}{mn}}+r_ns^2\Big(\frac{\log p}{n^*}\Big)^{(1-q)/2}+sr_n^2\right),
	\end{aligned}
	\end{equation*}
	which proves Theorem \ref{thm:nsm_conv}. Then Corollary \ref{cor:betaT_conv} can be easily obtained by applying the above formula iteratively.
\end{proof}

\begin{proof}[Proof of Theorem \ref{thm:normality}]
	By \eqref{eq:dist_debias_decomp} and Lemma \ref{lem:bound_precision}, \ref{lem:bound_H0} and \ref{lem:bound_Ub}, we have that
	\begin{align*}
		\bar{\vect{\beta}}^{(1)}-\vect{\beta}^{*}=&- \frac{1}{\hat{H}^{(T)}(0)}\hat{\vect{\Omega}}_{\mcH_1}\left\{\frac{1}{m}\sum_{j=1}^m\nabla\mcL_j(\vect{\beta}^{*})\right\} + O_{\mbP}\left(\frac{s^{3/2}\log^{(2-q)/2} p}{\sqrt{mn}(n^*)^{(1-q)/2}}+\frac{s\log p}{mn}\right)\\
			=&- \frac{1}{H(0)}\vect{\Omega}\left\{\frac{1}{m}\sum_{j=1}^m\nabla\mcL_j(\vect{\beta}^{*})\right\} + O_{\mbP}\left(\frac{s^{3/2}\log^{(2-q)/2} p}{\sqrt{mn}(n^*)^{(1-q)/2}}+\frac{s\log p}{mn}\right),\stepcounter{equation}\tag{\theequation}\label{eq:dist_clt_keyterm}
	\end{align*}
	where the remainders are bounded in $\ell_{\infty}$ norm. Notice that the major term is the sum of $mn$ i.i.d. variables, for each coordinate $l$ (where $l=1,\dots,p$), we know the variance is
	\begin{align*}
		\sigma_l^2 = \frac{1}{H^2(0)}\vect{\omega}_l^{\tp}\mbE\Big[\big\{f'(\vect{X}^{\tp}\vect{\beta}^*-Y)\big\}^2\vect{X}\vect{X}^{\tp}\Big]\vect{\omega}_l=\frac{\mbE\big[\big\{f'(\epsilon)\big\}^2\big]}{H^2(0)}\Omega_{l,l}.
	\end{align*}
	Therefore, by central limit theorem (see, e.g., Theorem 9.1.1 of \cite{chow_teicher.2012}), we have that 
	\begin{equation}	\label{eq:clt_main_term}
		-\frac{\sqrt{mn}}{\sigma_l}\frac{1}{H(0)}\vect{\Omega}_l\left\{\frac{1}{m}\sum_{j=1}^m\nabla\mcL_j(\vect{\beta}^{*})\right\}\xrightarrow{d}\mcN(0,1).
	\end{equation}
	When $s,n,n^*,p$ satisfies the constraints
	\begin{equation*}
		\frac{s^{3/2}\log p}{\sqrt{mnn^*}}+\frac{s\log p}{mn} = o\Big(\frac{1}{\sqrt{mn}}\Big)\Rightarrow s = o\Big( \frac{(n^*)^{(1-q)/3}}{(\log p)^{(2-q)/3}},\frac{\sqrt{mn}}{\log p}\Big),
	\end{equation*}
	we can substitute \eqref{eq:clt_main_term} into \eqref{eq:dist_clt_keyterm} and obtain that 
	\begin{equation*}
		\frac{\sqrt{mn}}{\sigma_l}(\bar{\beta}^{(T)}_l-\beta^*_l)\xrightarrow{d}\mcN(0,1),
	\end{equation*}
	which proves the theorem.
\end{proof}


\subsection{Proof of Theories in Section \ref{sec:glm_theory}}	\label{sec:glm_proof}

Throughout this section, we denote
\begin{equation}	\label{eq:weight_gradtilde}
	\vect{g}_j(\vect{\beta}_1,\vect{\beta}_2) = \frac{1}{n}\sum_{i\in\mcD_j}\frac{1}{\psi''(\vect{X}_i^{\tp}\vect{\beta}_1)}\Big\{-Y_i\vect{X}_i+\psi'(\vect{X}_i^{\tp}\vect{\beta}_2)\vect{X}_i\Big\},
\end{equation}
for convenience.
\begin{lemma}	\label{lem:b0b0_bound}
	Under Assumption \ref{assumpglm:X} to \ref{assumpglm:bound_below}, there hold
	\begin{equation*}
		\vect{g}_j(\hat{\vect{\beta}}^{(0)},\hat{\vect{\beta}}^{(0)}) - \vect{g}_j(\hat{\vect{\beta}}^{(0)},\vect{\beta}^{*}) = \hat{\vect{\Sigma}}_j(\hat{\vect{\beta}}^{(0)}-\vect{\beta}^*) + O_{\mbP}(\sqrt{\log p}r_n^2).
	\end{equation*}
\end{lemma}

\begin{proof}
	Note that
	\begin{align*}
		&\vect{g}_j(\hat{\vect{\beta}}^{(0)},\hat{\vect{\beta}}^{(0)}) - \vect{g}_j(\hat{\vect{\beta}}^{(0)},\vect{\beta}^{*})\stepcounter{equation}\tag{\theequation}\label{eq:gb0b0_diff}\\
		=& \frac{1}{n}\sum_{i\in\mcD_j}\frac{1}{\psi''(\vect{X}_i^{\tp}\hat{\vect{\beta}}^{(0)})}\Big\{\psi'(\vect{X}_i^{\tp}\hat{\vect{\beta}}^{(0)})\vect{X}_i-\psi'(\vect{X}_i^{\tp}\vect{\beta}^*)\vect{X}_i\Big\}\\
		=& \frac{1}{n}\sum_{i\in\mcD_j}\frac{1}{\psi''(\vect{X}_i^{\tp}\hat{\vect{\beta}}^{(0)})}\Big\{\psi''(\vect{X}_i^{\tp}\hat{\vect{\beta}}^{(0)})\vect{X}_i\vect{X}^{\tp}(\hat{\vect{\beta}}^{(0)}-\vect{\beta}^*)\\
		&+\int_0^1(1-u)\psi'''(\hat{\vect{\beta}}^{(0)}\vect{X}_i+u(\vect{\beta}^*-\hat{\vect{\beta}}^{(0)})\vect{X}_i) \diff u\big\{\vect{X}_i^{\tp}(\vect{\beta}^*-\hat{\vect{\beta}}^{(0)})\big\}^2\vect{X}_i\Big\}\\
		=& \hat{\vect{\Sigma}}_j(\hat{\vect{\beta}}^{(0)}-\vect{\beta}^*)+\int_0^1(1-u)\psi'''(\hat{\vect{\beta}}^{(0)}\vect{X}_i+u(\vect{\beta}^*-\hat{\vect{\beta}}^{(0)})\vect{X}_i) \diff u\big\{\vect{X}_i^{\tp}(\vect{\beta}^*-\hat{\vect{\beta}}^{(0)})\big\}^2\vect{X}_i.
	\end{align*}
	We first bound the term
	\begin{align*}
		\max_{i\in\mcD_j}\frac{1}{\psi''(\vect{X}_i^{\tp}\hat{\vect{\beta}}^{(0)})} \leq& \max_{i\in\mcD_j}\Big\{\frac{1}{\psi''(\vect{X}_i^{\tp}\vect{\beta}^{*})}+\frac{c_{\psi}|\vect{X}_i^{\tp}(\hat{\vect{\beta}}^{(0)}-\vect{\beta}^*)|}{\psi''(\vect{X}_i^{\tp}\vect{\beta}^*)\psi''(\vect{X}_i^{\tp}\hat{\vect{\beta}}^{(0)})}\Big\}\\
		\leq& \max_{i\in\mcD_j}\Big\{\frac{1}{\psi''(\vect{X}_i^{\tp}\vect{\beta}^{*})}+\frac{c_{\psi}|\vect{X}_i|_{\infty}|\hat{\vect{\beta}}^{(0)}-\vect{\beta}^*|_{1}}{\psi''(\vect{X}_i^{\tp}\vect{\beta}^*)^2 - \psi''(\vect{X}_i^{\tp}\vect{\beta}^*)c_{\psi}|\vect{X}_i|_{\infty}|\hat{\vect{\beta}}^{(0)}-\vect{\beta}^*|_{1}}\Big\}\\
		\leq& \max_{i\in\mcD_j}\Big\{\frac{2}{\psi''(\vect{X}_i^{\tp}\vect{\beta}^{*})}\Big\}\leq \frac{2}{c_l},
	\end{align*}
	where the last line follows from the fact that $|\hat{\vect{\beta}}^{(0)}-\vect{\beta}^*|_{1} = \sqrt{s}r_n=o(1/\sqrt{\log p})$, and the fact that $\max_{i\in\mcD_j}|\vect{X}_i|_{\infty}= O(\sqrt{\log p})$ by sub-gaussianity of $\vect{X}$ in Condition \ref{assump:X}. Therefore, we have that
	\begin{align*}
		&\Big|\frac{1}{n}\sum_{i\in\mcD_j}\frac{1}{\psi''(\vect{X}_i^{\tp}\hat{\vect{\beta}}^{(0)})}\int_0^1(1-u)\psi'''(\hat{\vect{\beta}}^{(0)}\vect{X}_i+u(\vect{\beta}^*-\hat{\vect{\beta}}^{(0)})\vect{X}_i) \diff u\big\{\vect{X}_i^{\tp}(\vect{\beta}^*-\hat{\vect{\beta}}^{(0)})\big\}^2\vect{X}_i\Big|_{\infty}\\
		\leq&\frac{c_{\psi}}{c_l}|\vect{\beta}^*-\hat{\vect{\beta}}^{(0)}|^2_2\sup_{\vect{v}\in\mbC(4s)}\Big|\frac{1}{n}\sum_{i\in\mcD_j}\big(\vect{X}_i^{\tp}\vect{v}\big)^2\vect{X}_i\Big|_{\infty}\\
		\leq&\frac{c_{\psi}}{c_l}r_n^2\max_{i\in\mcD_j}|\vect{X}_i|_{\infty}\sup_{\vect{v}\in\mbC(4s)}\frac{1}{n}\sum_{i\in\mcD_j}\big(\vect{X}_i^{\tp}\vect{v}\big)^2= O_{\mbP}(\sqrt{\log p}r_n^2),
	\end{align*}
	where $\mbC(s)=\{\vect{v}|\;|\vect{v}|_2\leq 1,|\vect{v}|_1\leq \sqrt{s}\}$, and by Theorem 1 of \cite{raskutti_etal.2010jmlr}, we know that
	\begin{equation*}
		\frac{1}{n}\sum_{i\in\mcD_j}\big(\vect{X}_i^{\tp}\vect{v}\big)^2 = O_{\mbP}(1).
	\end{equation*}
	Substitute it into \eqref{eq:gb0b0_diff} we can prove this lemma.
\end{proof}

\begin{lemma}	\label{lem:b0bs_bound}
	Under Assumption \ref{assumpglm:X} to \ref{assumpglm:bound_below}, there hold
	\begin{equation*}
		\vect{g}_j(\hat{\vect{\beta}}^{(0)},\vect{\beta}^{*}) - \vect{g}_j(\vect{\beta}^{*},\vect{\beta}^{*}) = O_{\mbP}(s\log^2pr_n^2).
	\end{equation*}
\end{lemma}

\begin{proof}
	We note that
	\begin{align*}
		&\vect{g}_j(\hat{\vect{\beta}}^{(0)},\vect{\beta}^{*}) - \vect{g}_j(\vect{\beta}^{*},\vect{\beta}^{*})\stepcounter{equation}\tag{\theequation}\label{eq:gdiff_decomp}\\
		=& \frac{1}{n}\sum_{i\in\mcD_j}\Big\{\frac{1}{\psi''(\vect{X}_i^{\tp}\hat{\vect{\beta}}^{(0)})} - \frac{1}{\psi''(\vect{X}_i^{\tp}\vect{\beta}^{*})}\Big\}\Big\{-Y_i\vect{X}_i+\psi'(\vect{X}_i^{\tp}\vect{\beta}^*)\vect{X}_i\Big\}\\
		=& \frac{1}{n}\sum_{i\in\mcD_j} \frac{\psi'''(\vect{X}_i^{\tp}\vect{\beta}^*)\vect{X}_i^{\tp}(\vect{\beta}^*-\hat{\vect{\beta}}^{(0)})}{\{\psi''(\vect{X}_i^{\tp}\vect{\beta}^{*})\}^2}\Big\{-Y_i\vect{X}_i+\psi'(\vect{X}_i^{\tp}\vect{\beta}^*)\vect{X}_i\Big\} + O_{\mbP}(s\log^2pr_n^2).
	\end{align*}
	We denote
	\begin{align*}
		\vect{U}(\vect{v})=& \frac{1}{n}\sum_{i\in\mcD_j} \frac{\psi'''(\vect{X}_i^{\tp}\vect{\beta}^*)\vect{X}_i^{\tp}\vect{v}}{\{\psi''(\vect{X}_i^{\tp}\vect{\beta}^{*})\}^2}\Big\{-Y_i\vect{X}_i+\psi'(\vect{X}_i^{\tp}\vect{\beta}^*)\vect{X}_i\Big\}.
	\end{align*}
	Define $\mcS=\big\{S_q\subseteq\{1,...,p\}\big| |S_q|_0=C_rs \big\}$ as the collection of index sets having $C_rs$ elements. Then we know that $|\mcS|=C_p^{C_rs}\leq p^{C_rs}$. For each $S_q\in\mcS$, we denote the parameter sets $\mcV_q=\{\vect{v}: \supp(\vect{v})\subseteq S_q,\; |\vect{v}|_2\leq 1\}$. For each $q$, we construct $\mfN_q$, an $1/2$-net of the set $\mcV_q$. By definition we know that $\mcV_q$ is a ball in $\mbR^{C_rs}$ with radius $1$. By Lemma 5.2 of \cite{vershynin.2010}, we have $\card(\mfN_q)\leq 5^{C_rs}$. Then we turn to bound
	\begin{align*}
		\max_{q}\sup_{\vect{v}\in\mcV_q}\big|\vect{U}(\vect{v})\big|_{\infty}\leq 2\max_{q}\max_{\vect{v}\in\mfN_q}\big|\vect{U}(\vect{v})\big|_{\infty}.
	\end{align*}
	For each $\vect{v}\in\mfN_q$ for some $q$, and each coordinate $l$, we have that
	\begin{align*}
		&\Big|\frac{\psi'''(\vect{X}_i^{\tp}\vect{\beta}^*)\vect{X}_i^{\tp}\vect{v}}{\{\psi''(\vect{X}_i^{\tp}\vect{\beta}^{*})\}^2}\Big\{-Y_iX_{i,l}+\psi'(\vect{X}_i^{\tp}\vect{\beta}^*)X_{i,l}\Big\}\Big|^{2/3}\\
		\leq&\Big(\frac{c_{\psi}}{c_l^2}\Big)^{2/3}\Big|\vect{X}_i^{\tp}\vect{v}\Big\{-Y_iX_{i,l}+\psi'(\vect{X}_i^{\tp}\vect{\beta}^*)X_{i,l}\Big\}\Big|^{2/3}\\
		\leq&\frac{1}{3}\Big(\frac{c_{\psi}}{c_l^2}\Big)^{2/3}\Big(|\vect{X}_i^{\tp}\vect{v}|^2+|Y_i-\psi'(\vect{X}_i^{\tp}\vect{\beta}^*)|^2+|X_{i,l}|^2\Big).
	\end{align*}
	Therefore, we have that each element of $U_l(\vect{v})$ follows $2/3$-sub-Weibull distribution. By Theorem 3.1 of \cite{kuchibhotla_chakrabortty.2018}, we have that
	\begin{align*}
		\max_q\max_{\vect{v}\in\mfN_q}\mbP\Big(\Big|\vect{U}(\vect{v})\Big|_{\infty}\geq C_1\Big(\sqrt{\frac{s\log p}{n}}+\frac{(s\log p)^{3/2}}{n}\Big)\Big)\leq p^{-\gamma s},
	\end{align*}
	for some large constants $C_1,\gamma>0$. Therefore we can obtain that
	\begin{align*}
		&\mbP\Big(\max_{q}\sup_{\vect{v}\in\mcV_q}\Big|\vect{U}(\vect{v})\Big|_{\infty}\geq 2C_1\Big(\sqrt{\frac{s\log p}{n}}+\frac{(s\log p)^{3/2}}{n}\Big)\Big)\\
		\leq&p^{C_rs}5^{C_rs}\max_{q}\max_{\vect{v}\in\mfN_q}\mbP\Big(\Big|\vect{U}(\vect{v})\Big|_{\infty}\geq C_1\Big(\sqrt{\frac{s\log p}{n}}+\frac{(s\log p)^{3/2}}{n}\Big)\Big)\leq p^{-\tilde{\gamma}s}.
	\end{align*}
	Substitute it into \eqref{eq:gdiff_decomp} we have
	\begin{equation*}
		\vect{g}_j(\hat{\vect{\beta}}^{(0)},\vect{\beta}^{*}) - \vect{g}_j(\vect{\beta}^{*},\vect{\beta}^{*})=O_{\mbP}\Big(r_n\sqrt{\frac{s\log p}{n}}+s\log^2pr_n^2\Big),
	\end{equation*}
	which proves the lemma.
\end{proof}

\begin{lemma}	\label{lem:event_const}
	Let $X_1,\dots,X_n$ be random variables and $\mfX_1,\dots,\mfX_n$ be a sequence of event. Further assume $h$ to be a measurable function of the tuple $(X_1,\dots,X_n)$ taking values in $\mbR^p$, and $\mfS$ be a measurable set in $\mbR^p$. Then we have that
	\begin{align*}
		\Big\{h\big(X_1\mbI(\mfX_1),\dots,X_n\mbI(\mfX_n)\big)\in \mfS; \cap_{i=1}^n\mfX_i\Big\}\subseteq\Big\{h\big(X_1,\dots,X_n\big)\in \mfS; \cap_{i=1}^n\mfX_i\Big\}.
	\end{align*}
\end{lemma}

\begin{lemma}	\label{lem:rewgrad_concen}
	Under Condition \ref{assumpglm:X} to \ref{assumpglm:bound_below}, let $\mfX_i = \{\vect{X}_i|\psi''(\vect{X}_i^{\tp}\vect{\beta}^*)\geq c_l\}$, then we have that
	\begin{equation*}
		\mbP\Big(\Big|\frac{1}{m}\sum_{j=1}^m\vect{g}_j(\vect{\beta}^*,\vect{\beta}^*)\Big|_{\infty}\leq C_1\sqrt{\frac{\log p}{mn}},\cap\mfX_i\Big)\geq 1-O(p^{\gamma_1}).
	\end{equation*}
\end{lemma}

\begin{proof}
	By definition, $\frac{1}{m}\sum_{j=1}^m\vect{g}_j(\vect{\beta}^*,\vect{\beta}^*)$ is the average of $mn$ i.i.d. terms $\frac{1}{\psi''(\vect{X}_i^{\tp}\vect{\beta}^*)}\Big\{-Y_i\vect{X}_i+\psi'(\vect{X}_i^{\tp}\vect{\beta}^*)\vect{X}_i\Big\}$. By Lemma \ref{lem:event_const}, we have that
	\begin{align*}
		&\mbP\Big(\Big|\frac{1}{m}\sum_{j=1}^m\vect{g}_j(\vect{\beta}^*,\vect{\beta}^*)\Big|_{\infty}\leq C_1\sqrt{\frac{\log p}{mn}},\cap\mfX_i\Big)\\
		\geq&\mbP\Big(\Big|\frac{1}{mn}\sum_{i=1}^{mn}\frac{1}{\psi''(\vect{X}_i^{\tp}\vect{\beta}^*)}\Big\{-Y_i\vect{X}_i+\psi'(\vect{X}_i^{\tp}\vect{\beta}^*)\vect{X}_i\Big\}\mbI(\mfX_i)\Big|_{\infty}\leq C_1\sqrt{\frac{\log p}{mn}},\cap\mfX_i\Big)\\
		\geq&1-p^{-\gamma} - \mbP\Big(\Big|\frac{1}{mn}\sum_{i=1}^{mn}\frac{1}{\psi''(\vect{X}_i^{\tp}\vect{\beta}^*)}\Big\{-Y_i\vect{X}_i+\psi'(\vect{X}_i^{\tp}\vect{\beta}^*)\vect{X}_i\Big\}\mbI(\mfX_i)\Big|_{\infty}\geq C_1\sqrt{\frac{\log p}{mn}}\Big).
	\end{align*}
	For each $i$, we have that
	\begin{align*}
		&\sup_{\vect{v}\in\mbS^{p-1}}\mfE\Big\{ \frac{1}{\psi''(\vect{X}_i^{\tp}\vect{\beta}^*)}\vect{v}^{\tp}\Big\{-Y_i\vect{X}_i+\psi'(\vect{X}_i^{\tp}\vect{\beta}^*)\vect{X}_i\Big\}\mbI(\mfX_i), c_l\kappa_0/2\Big\}\\
		\leq&\sup_{\vect{v}\in\mbS^{p-1}}\mfE\Big\{ \frac{1}{c_l}\vect{v}^{\tp}\Big\{-Y_i\vect{X}_i+\psi'(\vect{X}_i^{\tp}\vect{\beta}^*)\vect{X}_i\Big\}, c_l\kappa_0/2 \Big\}\\
		\leq&\frac{4}{c_l^2\kappa_0^2}\sup_{\vect{v}\in\mbS^{p-1}}\mbE\Big[\exp \kappa_0\vect{v}^{\tp}\Big\{-Y_i\vect{X}_i+\psi'(\vect{X}_i^{\tp}\vect{\beta}^*)\vect{X}_i\Big\} \Big]\\
		\leq&\frac{2}{c_l^2\kappa_0^2}\Big\{\sup_{\vect{v}\in\mbS^{p-1}}\mbE\Big[\exp\Big\{ \kappa_0|\vect{v}^{\tp}\vect{X}_i|^2\Big\} \Big]+\mbE\Big[\exp\Big\{\kappa_0|\psi'(\vect{X}_i^{\tp}\vect{\beta}^*)-Y_i|^2\Big\} \Big]\Big\}\\
		\leq&\frac{4}{c_l^2\kappa_0^2}C_0.
	\end{align*}
	Therefore, we can apply Lemma \ref{lemma:cai_liu} and obtain that
	\begin{equation*}
		\mbP\Big(\Big|\frac{1}{mn}\sum_{i=1}^{mn}\frac{1}{\psi''(\vect{X}_i^{\tp}\vect{\beta}^*)}\Big\{-Y_i\vect{X}_i+\psi'(\vect{X}_i^{\tp}\vect{\beta}^*)\vect{X}_i\Big\}\mbI(\mfX_i)\Big|_{\infty}\geq C_1\sqrt{\frac{\log p}{mn}}\Big)=O(p^{-\gamma_1}),
	\end{equation*}
	which proves the lemma.
\end{proof}

\begin{proof}[Proof of Theorem \ref{thm:glm_conv} and Corollary \ref{cor:glmT_conv}]
	From \eqref{eq:dss_glm_tilde}, we have the following decomposition
	\begin{align*}	
	\bar{\vect{\beta}}^{(1)}-\vect{\beta}^{*}=&\hat{\vect{\beta}}^{(0)} - \hat{\vect{\Omega}}_{\mcH_1}\frac{1}{m}\sum_{j=1}^m\vect{g}_j(\hat{\vect{\beta}}^{(0)}, \hat{\vect{\beta}}^{(0)} )-\vect{\beta}^*\stepcounter{equation}\tag{\theequation}\label{eq:dssglm_decomp}\\
		=&\hat{\vect{\beta}}^{(0)}-\vect{\beta}^* - \hat{\vect{\Omega}}_{\mcH_1}\Bigg\{ \frac{1}{m}\sum_{j=1}^m\big\{\vect{g}_j(\hat{\vect{\beta}}^{(0)}, \hat{\vect{\beta}}^{(0)} )-\vect{g}_j(\hat{\vect{\beta}}^{(0)}, \vect{\beta}^* )\big\}+\frac{1}{m}\sum_{j=1}^m\big\{\vect{g}_j(\hat{\vect{\beta}}^{(0)}, \vect{\beta}^*)-\vect{g}_j(\vect{\beta}^*, \vect{\beta}^* )\big\}\\
		&\qquad\mbox{}\qquad\mbox{}\qquad\mbox{}\qquad+\frac{1}{m}\sum_{j=1}^m\vect{g}_j(\vect{\beta}^*, \vect{\beta}^* ) \Bigg\}.
	\end{align*}
	By Lemma \ref{lem:b0b0_bound} and \ref{lem:b0bs_bound}, we have that \eqref{eq:dssglm_decomp} equals to
	\begin{align*}
		&\hat{\vect{\beta}}^{(0)}-\vect{\beta}^* - \hat{\vect{\Omega}}_{\mcH_1}\left\{ \hat{\vect{\Sigma}}_N(\hat{\vect{\beta}}^{(0)}-\vect{\beta}^*)+ \frac{1}{m}\sum_{j=1}^m\vect{g}_j(\vect{\beta}^*, \vect{\beta}^* )+O_{\mbP}(s\log^2pr_n^2)  \right\} \\
		=&\big(\vect{I} - \hat{\vect{\Omega}}_{\mcH_1}\hat{\vect{\Sigma}}_N\big)(\hat{\vect{\beta}}^{(0)}-\vect{\beta}^*) - \hat{\vect{\Omega}}_{\mcH_1}\left\{ \frac{1}{m}\sum_{j=1}^m\vect{g}_j(\vect{\beta}^*, \vect{\beta}^* )  \right\} +O_{\mbP}(s\log^2pr_n^2)\\
		=&\big(\vect{\Omega} - \hat{\vect{\Omega}}_{\mcH_1}\big)\vect{\Sigma}(\hat{\vect{\beta}}^{(0)}-\vect{\beta}^*) +  \hat{\vect{\Omega}}_{\mcH_1}\big(\vect{\Sigma}-\hat{\vect{\Sigma}}_N\big)(\hat{\vect{\beta}}^{(0)}-\vect{\beta}^*) - \hat{\vect{\Omega}}_{\mcH_1}\left\{ \frac{1}{m}\sum_{j=1}^m\vect{g}_j(\vect{\beta}^*, \vect{\beta}^* )  \right\} +O_{\mbP}(s\log^2pr_n^2).
	\end{align*}
	Apply Lemma \ref{lem:rewgrad_concen} and Lemma \ref{lem:bound_precision} we have that
	\begin{align*}
		\big|\bar{\vect{\beta}}^{(1)}-\vect{\beta}^{*}\big|_{\infty} =& \Big|\hat{\vect{\Omega}}_{\mcH_1}\Big\{ \frac{1}{m}\sum_{j=1}^m\vect{g}_j(\vect{\beta}^*, \vect{\beta}^* )  \Big\}\Big|_{\infty} + O_{\mbP}\Big(sr_n\Big(\frac{\log p}{n^*}\Big)^{(1-q)/2} + r_n\sqrt{\frac{s\log p}{mn}} + s\log^2pr_n^2\Big)\\
			=&O_{\mbP}\Big(\sqrt{\frac{\log p}{mn}} + sr_n\sqrt{\frac{\log p}{n^*}} + r_n\sqrt{\frac{s\log p}{mn}} + s\log^2pr_n^2\Big).
	\end{align*}
	When the threshold level $\tau$ is larger than $\big|\bar{\vect{\beta}}^{(1)}-\vect{\beta}^{*}\big|_{\infty}$, we know that $|\bar{\beta}^{(1)}_l|<\tau$ for $l\notin S$, therefore $\hat{\beta}_l^{(1)}=0$. For $l\in S$, there is 
	\begin{equation*}
		\big|\hat{\beta}^{(1)}_l-\beta_l^{*}\big|\leq \big|\bar{\beta}^{(1)}_l-\beta_l^{*}\big|+\tau = O_{\mbP}\Big(\sqrt{\frac{\log p}{mn}} + r_ns\Big(\frac{\log p}{n^*}\Big)^{(1-q)/2} + r_n\sqrt{\frac{s\log p}{mn}} + s\log^2pr_n^2\Big).
	\end{equation*}
	Therefore, 
	\begin{align*}
		|\hat{\vect{\beta}}^{(1)}-\vect{\beta}^*|_1 =& |\hat{\vect{\beta}}^{(1)}_S-\vect{\beta}^*_S|_1 + |\hat{\vect{\beta}}^{(1)}_{S^c}-\vect{\beta}^*_{S^c}|_1\leq s |\hat{\vect{\beta}}^{(1)}_S-\vect{\beta}^*_S|_{\infty}\\
		=&O_{\mbP}\Big(s\sqrt{\frac{\log p}{mn}} + s^2r_n\Big(\frac{\log p}{n^*}\Big)^{(1-q)/2} + r_ns\sqrt{\frac{s\log p}{mn}} + s^2\log^2pr_n^2\Big).
	\end{align*} 
	Similarly, we can prove that
	\begin{equation*}	
	\begin{aligned}
		\big|\hat{\vect{\beta}}^{(1)}-\vect{\beta}^*\big|_{2}=&O_{\mbP}\Big(\sqrt{\frac{s\log p}{mn}} + r_ns^{3/2}\Big(\frac{\log p}{n^*}\Big)^{(1-q)/2} + r_n\sqrt{\frac{s^2\log p}{mn}} + s^{3/2}\log^2pr_n^2\Big),
	\end{aligned}
	\end{equation*}
	which proves Theorem \ref{thm:glm_conv}. Then Corollary \ref{cor:glmT_conv} can be obtained by applying the above rate iteratively.
\end{proof}

\begin{proof}[Proof of Theorem \ref{thm:glm_normality}]
	From \eqref{eq:dssglm_decomp}, when $r_n = O_{\mbP}\sqrt{s\log p/(mn)}$, we have that 
	\begin{align*}
		\bar{\vect{\beta}}^{(1)}-\vect{\beta}^{*} = - \vect{\Omega}\left\{ \frac{1}{m}\sum_{j=1}^m\vect{g}_j(\vect{\beta}^*, \vect{\beta}^* )  \right\}+O_{\mbP}\Big(\frac{s^2\log^3 p}{mn}+\frac{s^{3/2}\log^{(2-q)/2} p}{\sqrt{mn}(n^*)^{(1-q)/2}}\Big).
	\end{align*}
	The major term $\frac{1}{m}\sum_{j=1}^m\vect{g}_j(\vect{\beta}^*,\vect{\beta}^*)$ is the average of $N$ i.i.d. terms of $\frac{1}{\psi''(\vect{X}_i^{\tp}\vect{\beta}^*)}\Big\{-Y_i\vect{X}_i+\psi'(\vect{X}_i^{\tp}\vect{\beta}^*)\vect{X}_i\Big\}$. We know the covariance matrix is $\mbE[\vect{X}\vect{X}^{\tp}/\psi''(\vect{X}^{\tp}\vect{\beta}^*)]$. Then for each coordinate $l$ (where $l=1,\dots,p$), we denote the variance
	\begin{equation*}
		\sigma_l^2 = \vect{\omega}_l^{\tp}\mbE\Big[\frac{1}{\psi''(\vect{X}^{\tp}\vect{\beta}^*)}\vect{X}\vect{X}^{\tp}\Big]\vect{\omega}_l,
	\end{equation*}
	by the central limit theorem, we have that 
	\begin{equation}	\label{eq:clt_glm_main}
		-\frac{\sqrt{mn}}{\sigma_l}\vect{\Omega}_l\left\{\frac{1}{m}\sum_{j=1}^m\nabla\mcL_j(\vect{\beta}^{*})\right\}\xrightarrow{d}\mcN(0,1).
	\end{equation}
	When there are constraints 
	\begin{equation*}
		s = o\Big(\frac{(mn)^{1/4}}{\log^{3/2}p},\frac{(n^*)^{(1-q)/3}}{\log^{(2-q)/3}p}\Big),
	\end{equation*}
	we can substitute \eqref{eq:clt_glm_main} to \eqref{eq:dssglm_decomp} and obtain that
	\begin{equation*}
		\frac{\sqrt{mn}}{\sigma_l}(\bar{\beta}^{(T)}_l-\beta^*_l)\xrightarrow{d}\mcN(0,1),
	\end{equation*}
	which proves the theorem.
\end{proof}


\subsection{Examples Verification}	\label{sec:exp_verif}

\begin{proposition} \label{prop:huber_verify}
    In the Huber regression model defined in Example \ref{exp:huber}, assume the covariate $\vect{X}$ satisfies sub-gaussian property, \ie,
	\begin{equation*}
		\sup_{\vect{v}\in\mbS^{p-1}}\mbE\big\{\exp(\eta|\vect{v}^{\tp}\vect{X}|^2)\big\}\leq C_M,
	\end{equation*}
	for some constants $\eta,C_M>0$. Moreover, we assume that the noise $\epsilon$ has probability density function $f(x)$ uniformly bounded by a constant $C_d$. Then Huber loss satisfies Condition \ref{assump:fpp_mfE} and \ref{assump:smooth}.
\end{proposition}

\begin{proof}
    The proof can be found in Proposition 1 of the supplementary materials in \cite{tu2021byzantine}.
\end{proof}

\begin{proposition} \label{prop:med_verify}
    In the Huber regression model defined in Example \ref{exp:huber}, assume the covariate $\vect{X}$ satisfies sub-gaussian property, \ie,
	\begin{equation*}
		\sup_{\vect{v}\in\mbS^{p-1}}\mbE\big\{\exp(\eta|\vect{v}^{\tp}\vect{X}|^2)\big\}\leq C_M,
	\end{equation*}
	for some constants $\eta,C_M>0$. Moreover, we assume that the noise $\epsilon$ has probability density function $f(x)$ uniformly bounded by a constant $C_d$. Then absolute deviation loss satisfies Conditions \ref{assump:fpp_mfE} and \ref{assump:non-smooth}.
\end{proposition}

\begin{proof}
    The proof can be found in Proposition 2 of the supplementary materials in \cite{tu2021byzantine}.
\end{proof}

\begin{proposition} \label{prop:logit_verify}
    In the logistic regression model defined in Example \ref{exp:logit}, assume the covariate $\vect{X}$ is uniformly bounded, \ie,
    \begin{equation*}
        |\vect{X}|_2\leq C_M,
    \end{equation*}
    for some constant $C_M>0$. Then the logistic loss satisfies Conditions \ref{assumpglm:subgauss}, \ref{assumpglm:lip_link}, and \ref{assumpglm:bound_below}.
\end{proposition}

\begin{proof}
    Since 
    \begin{align*}
        \psi'(x) =&\frac{1}{1+\exp(-x)},\\
        \psi''(x)=&\frac{1}{(1+\exp(x))(1+\exp(-x))},\\
        \psi'''(x) = &\frac{1}{(1+\exp(x))^2(1+\exp(-x))} - \frac{1}{(1+\exp(x))(1+\exp(-x))^2}.
    \end{align*}
    We can see that 
    \begin{align*}
        &\sup_{x\in\mbR}\max\{|\psi''(x)|,|\psi'''(x)|\}\leq 1,\\
        &|\psi'(\vect{X}^{\tp}\vect{\beta}^*)-Y|\leq 2.
    \end{align*}
    Combining with the boundedness condition of $\vect{X}$, we have that Conditions \ref{assumpglm:subgauss}, \ref{assumpglm:lip_link}, and \ref{assumpglm:bound_below} are sufficed.
\end{proof}

\end{document}